\documentclass[letterpaper]{article} 
\usepackage{aaai2026}  
\usepackage{times}  
\usepackage{helvet}  
\usepackage{courier}  
\usepackage[hyphens]{url}  
\usepackage{graphicx} 
\usepackage{natbib}  
\usepackage{caption} 
\usepackage{algorithm}
\usepackage{algorithmic}

\usepackage{newfloat}
\usepackage{listings}
\DeclareCaptionStyle{ruled}{labelfont=normalfont,labelsep=colon,strut=off} 
\floatstyle{ruled}
\newfloat{listing}{tb}{lst}{}
\floatname{listing}{Listing}
\input{defns.sty}
\usepackage[mathscr]{eucal}
\usepackage{epsfig,epsf,psfrag}
\usepackage{amssymb,amsmath,amsthm,amsfonts,latexsym}
\usepackage{amsmath,graphicx,bm,url}
\usepackage{mathrsfs}
\usepackage{epstopdf}
\newtheorem{theorem}{Theorem}
\usepackage{caption}
\usepackage{subcaption}

\usepackage{pgfplots}
\usepgfplotslibrary{groupplots}
\pgfplotsset{compat=newest}
\newtheorem{lemma}[theorem]{Lemma}

\usepackage{dsfont}
\usepackage{booktabs}
\newtheorem{corollary}{Corollary}[theorem]

\catcode`~=11 \def\UrlSpecials{\do\~{\kern -.15em\lower .7ex\hbox{~}\kern .04em}} \catcode`~=13 

\allowdisplaybreaks[1]

\DeclareMathOperator*{\argmax}{arg\,max}

\newcommand{\calA}{\mathcal{A}}

\newcommand{\calE}{\mathcal{E}}
\newcommand{\calF}{\mathcal{F}}

\newcommand{\calH}{\mathcal{H}}
\newcommand{\calI}{\mathcal{I}}

\newcommand{\calN}{\mathcal{N}}
\newcommand{\calO}{\mathcal{O}}

\newcommand{\calW}{\mathcal{W}}

\newcommand{\rmd}{\mathrm{d}}

\newcommand{\rmK}{\mathrm{K}}

\newcommand{\rmL}{\mathrm{L}}

\newcommand{\bbE}{\mathbb{E}}

\newcommand{\bbN}{\mathbb{N}}

\newcommand{\bbP}{\mathbb{P}}

\DeclareMathAlphabet{\mathbsf}{OT1}{cmss}{bx}{n}
\DeclareMathAlphabet{\mathssf}{OT1}{cmss}{m}{sl}

\DeclareSymbolFont{bsfletters}{OT1}{cmss}{bx}{n}  
\DeclareSymbolFont{ssfletters}{OT1}{cmss}{m}{n}
\DeclareMathSymbol{\bsfGamma}{0}{bsfletters}{'000}
\DeclareMathSymbol{\ssfGamma}{0}{ssfletters}{'000}
\DeclareMathSymbol{\bsfDelta}{0}{bsfletters}{'001}
\DeclareMathSymbol{\ssfDelta}{0}{ssfletters}{'001}
\DeclareMathSymbol{\bsfTheta}{0}{bsfletters}{'002}
\DeclareMathSymbol{\ssfTheta}{0}{ssfletters}{'002}
\DeclareMathSymbol{\bsfLambda}{0}{bsfletters}{'003}
\DeclareMathSymbol{\ssfLambda}{0}{ssfletters}{'003}
\DeclareMathSymbol{\bsfXi}{0}{bsfletters}{'004}
\DeclareMathSymbol{\ssfXi}{0}{ssfletters}{'004}
\DeclareMathSymbol{\bsfPi}{0}{bsfletters}{'005}
\DeclareMathSymbol{\ssfPi}{0}{ssfletters}{'005}
\DeclareMathSymbol{\bsfSigma}{0}{bsfletters}{'006}
\DeclareMathSymbol{\ssfSigma}{0}{ssfletters}{'006}
\DeclareMathSymbol{\bsfUpsilon}{0}{bsfletters}{'007}
\DeclareMathSymbol{\ssfUpsilon}{0}{ssfletters}{'007}
\DeclareMathSymbol{\bsfPhi}{0}{bsfletters}{'010}
\DeclareMathSymbol{\ssfPhi}{0}{ssfletters}{'010}
\DeclareMathSymbol{\bsfPsi}{0}{bsfletters}{'011}
\DeclareMathSymbol{\ssfPsi}{0}{ssfletters}{'011}
\DeclareMathSymbol{\bsfOmega}{0}{bsfletters}{'012}
\DeclareMathSymbol{\ssfOmega}{0}{ssfletters}{'012}

\newcommand{\qednew}{\nobreak \ifvmode \relax \else
      \ifdim\lastskip<1.5em \hskip-\lastskip
      \hskip1.5em plus0em minus0.5em \fi \nobreak
      \vrule height0.75em width0.5em depth0.25em\fi}

\title{Constrained Best Arm Identification with Tests for Feasibility}
\author {
    Ting Cai,
    Kirthevasan Kandasamy
}
\affiliations {
    University of Wisconsin-Madison\\
    tingcai@cs.wisc.edu, kandasamy@cs.wisc.edu
}

\begin{document}

\maketitle

\begin{abstract}
Best arm identification (BAI) aims to identify the highest-performance arm among a set of $K$ arms by collecting stochastic samples from each arm.
In real-world problems, the best arm needs to satisfy additional feasibility constraints.
While there is limited prior work on BAI with feasibility constraints, they typically assume the performance and constraints are observed simultaneously on each pull of an arm.
However, this assumption does not reflect most practical use cases, e.g., in drug discovery, we wish to find the most potent drug whose toxicity and solubility are below certain safety thresholds.
These safety experiments can be
conducted separately from the potency measurement.
Thus, this requires designing BAI algorithms that not only decide which arm to pull
but also decide whether to test for the arm's performance or feasibility.
In this work, we study feasible BAI which allows a
decision-maker to choose a tuple $(i,\ell)$, where $i\in [K]$ denotes an arm and
$\ell$ denotes whether she wishes to test for its performance ($\ell=0$) or any of its $N$
feasibility constraints ($\ell\in[N]$). 
We focus on the fixed confidence setting, which is to identify the
\textit{feasible} arm with the \textit{highest performance}, with a probability of at least
$1-\delta$.
We propose an efficient algorithm and upper-bound its sample complexity,
showing our algorithm can naturally adapt to the problem's difficulty and eliminate arms by worse performance or infeasibility, whichever is easier. 
We complement this upper bound with a lower bound
showing that
our algorithm is \textit{asymptotically ($\delta\rightarrow 0$) optimal}.
Finally, we empirically show that our algorithm outperforms other state-of-the-art BAI algorithms
in both synthetic and real-world datasets.
\end{abstract}


\newcommand{\insertAlgoOurs}{%
\begin{algorithm}[t!]
  \caption{ }  \label{alg:ours}
  \textbf{Input}: $[K]$ arms, $\delta\in(0,1)$.\\
  \textbf{Parameter}: $S \leftarrow [K]$, $P \leftarrow [K]$, $F \leftarrow \emptyset$, $I \leftarrow \emptyset$, $H_{i} \leftarrow [N]$ $\forall i\in[K]$. 
  \begin{algorithmic}[1]
    \FOR {$t = 1,2, \cdots$} 
        \IF {$S = \emptyset$} 
            \STATE \textbf{Return} $K+1$.
        \ELSIF {$P = \{i\}$ is a singleton \textbf{and} $i\in F$ }  
                \STATE \textbf{Return} $i$  
        \ELSIF {$P = \{i\}$ is a singleton \textbf{and} $i\notin F$ }       
            \STATE  $\sampleforsafety(i)$   
        \ELSE
            \STATE $a_t \leftarrow \argmax_{i\in P} \widehat{\mu}_{i,0}(t)$
            \STATE $b_t \leftarrow \argmax_{i\in P\backslash\{a_t\}} \overline{\mu}_{i, 0}(t)$
            \STATE Sample performance of $a_t$ and $b_t$.
            \IF{$a_t \notin F$}
                \STATE $\sampleforsafety(a_t)$
                
            \ENDIF
            \IF{$b_t \notin F$}
                \STATE $\sampleforsafety(b_t)$
            \ENDIF
        \ENDIF
        \STATE $S \leftarrow S \backslash I$
        \IF{$F \neq \emptyset$}
            \STATE $S \leftarrow \{i \in S:  \overline{\mu}_{i,0}(t) > \max_{j\in F} \underline{\mu}_{j,0}(t) \}$ 
        \ENDIF
        \STATE $P \leftarrow \{i \in S: \overline{\mu}_{i,0}(t) > \max_{j\in S} \underline{\mu}_{j,0}(t) \}$
    \ENDFOR

    \vspace{0.10in}
    {\textbf{Function:} \sampleforsafety($i$)}
        \STATE
            $\ell_t \leftarrow \argmax_{\ell \in H_{i}} \widetilde{\mu}_{i,\ell}(t)$
        \STATE
                Sample constraint $\ell_t$ of arm $i$
        \IF{$\underline{\mu}_{i, \ell_t}(t) > 1/2$}
            \STATE $I \leftarrow I \cup \{i\}$
        \ELSIF{$\overline{\mu}_{i, \ell_t}(t) < 1/2$}
            \STATE $H_{i} \leftarrow H_{i} \backslash \{\ell_t\}$
        \IF{$H_{i} = \emptyset$}
            \STATE $F \leftarrow F \cup \{i\}$
            \ENDIF
        \ENDIF
  \end{algorithmic}
\end{algorithm}
\vspace{-0.1in}
}

\section{INTRODUCTION}
\label{sec:intro}
Best arm identification (BAI) contains $K$ arms that allow a decision maker to repeatedly pull
one of the arms and observes an i.i.d. sample drawn from the distribution associated with that arm~\cite{bechhofer1958sequential,paulson1964sequential}.
In the fixed confidence setting of BAI, given a target failure rate $\delta>0$, the decision-maker aims to identify the arm with the \textit{highest expected performance},
with probability at least $1-\delta$,
while keeping the number of pulls to a minimum.
This problem has been widely applied in areas such as drug discovery, crowdsourcing,
distributed systems, and A/B testing~\cite{jun2016top,koenig1985procedure,schmidt2006integrating,van2017automatic,zhou2014optimal}.

In many real-world applications, we additionally wish to find the optimal arm that satisfies \textit{feasibility constraints}, which cannot be modeled analytically and need to be evaluated
via costly experimentation.
Moreover, these feasibility experiments can usually be tested separately from the performance evaluation.
As a motivating example, in drug discovery, we 
wish to find the most potent drug,
but also need to guarantee that the drug is soluble in the bloodstream,
and the risk of adverse side effects is small.
Scientists have developed different experiments
to measure the drug's potency (performance),
solubility (feasibility),
and toxicity (feasibility), which can be separately carried out independent of each
other~\cite{thall1998strategy}.
Similarly, in database tuning applications~\cite{van2017automatic}, we wish to 
minimize the end-to-end latency,
while guaranteeing the risk of system-wide failures is small.
The latency and robustness can also be evaluated via separate
tests~\cite{kanellis2020too}.

Majority of the prior work on BAI does not study feasibility constraints.
The only work on feasible BAI~\cite {katz2018feasible,katz2019top}
assumes that the performance and feasibility tests are all conducted simultaneously, which does not hold in many real-world settings, such as the examples highlighted above
Often, we can conduct separate experiments
and only observe the result of the experiment conducted.
Thus, applying methods in prior works to our setting by naively pulling all arms simultaneously can be unnecessarily
expensive as it samples all performance and feasibility distributions when testing one arm and does not focus on the most important experiment to determine an arm's optimality and/or feasibility.
Moreover, from a theoretical perspective, simultaneous pulls cannot capture the \textit{actual complexity} of the number of individual tests.

In our paper, we introduce a novel BAI formalism.
Given $K$ arms, each arm is associated with $N+1$ distributions.
For $i\in[K]$ and $\ell\in \{0\}\cup [N]$,
$\mu_{i,\ell}\in [0,1]$ denotes the unknown mean of arm $i$'s $\ell^{\rm th}$ distribution.
On each round, the decision-maker chooses a tuple $(i,\ell)$ where $i\in[K]$ denotes the arm
and $\ell$ denotes whether she wishes to test for its performance ($\ell=0$) or its $\ell^{\rm th}$
feasibility constraint ($\ell\in[N]$).
The thresholds for each feasibility constraint are given and for simplicity\footnote{It can be easily extended to different thresholds for different constraints.}, we assume the thresholds are all $\frac{1}{2}$.
An arm is said to be \emph{feasible} if $\mu_{i,\ell} < \frac{1}{2}$ for all $\ell\in[N]$.
Given $\delta \in (0,1)$, the goal of the decision-maker is to identify the optimal arm, i.e. the
feasible arm with the highest $\mu_{i,0}$, with probability at least $1-\delta$,
while minimizing the total number of collected samples. 

\textit{The \textbf{key challenge} to solve this problem is to balance when to test for feasibility and when to test for performance for different types of suboptimal arms at the same time.} We illustrate it using two naive algorithms:

First, consider a naive two-stage algorithm that identifies all feasible arms first and then executes a BAI algorithm on these arms. This algorithm is inefficient on suboptimal arms whose feasibility means $\mu_{i,\ell}$ for $\ell \in [N]$ are close to $\frac{1}{2}$ but $\mu_{i,0}$ value is much smaller compared to the optimal arm. Because it will waste many samples\footnote{%
Recall that, given two sub-Gaussian distributions with means $\mu$ and $\mu'$, and a threshold
 $\xi$,
we require $O((\mu - \mu')^{-2})$ samples to decide if $\mu<\mu'$ or $\mu>\mu'$,
and $O((\mu - \xi)^{-2})$ samples to decide if $\mu<\xi$ or $\mu>\xi$.
} on testing feasibility while these arms could have been eliminated faster if we started by comparing performance. 

Consider another algorithm that first tries to identify the best-performance arm and then tests its feasibility. It will eliminate the arm if it is infeasible and repeat with the remaining arms until it finds a feasible arm. This algorithm is inefficient on another type of suboptimal arms, whose $\mu_{i,0}$ values are slightly larger than the optimal arm but are clearly infeasible, i.e., $\mu_{i,\ell}\gg\frac{1}{2}$ for $\ell\in
[N]$. Because it will waste many samples to differentiate the performance of these arms which could
have been easily eliminated by feasibility.
Note that these two algorithms could indeed efficiently eliminate the two types of suboptimal arms described for each other.

In summary, the \textbf{main contributions} of this paper are:

\emph{1. Problem formalism and lower bound}: We define a novel feasible BAI problem and first quantify the complexity of the problem by developing a novel complexity term for each type of arm, which captures the best way to eliminate each type of suboptimal arm and the cost necessary to identify the optimal arm. Leveraging the insights from the complexity terms, we then provide a gap-dependent lower bound on the total expected sample complexity.
    
\emph{2. Algorithm design and upper bound}: 
We propose an algorithm which tests performance and feasibility for each arm simultaneously but only tests its performance and/or at most one of its $N$ feasibility constraints, which can eliminate all suboptimal arms in the easiest way.
Moreover, the number of samples collected by the algorithm does not scale linearly with the number of feasibility constraints $N$.  
We then provide the upper bound on the expected sample complexity of the algorithm, which is shown to be asymptotically ($\delta \to 0$) optimal compared with the lower bound.

\emph{3. Experiments}: We empirically compare our algorithm with other state-of-the-art algorithms on both synthetic and real-world datasets from drug discovery. In all the experiments, our algorithm outperforms all other algorithms.

\subsection{Related Work}
\label{sec:relatedwork}
The bandit framework is a popular paradigm to study
exploration-exploitation tradeoffs that occur in sequential decision-making under
uncertainty~\cite{lai85bandits,thompson33sampling}.
Here, a decision-maker adaptively samples one of $K$ arms from a bandit model so as to
achieve a given objective.
There is extensive prior work on developing algorithms for
BAI, whose goal is to identify the arm
with the highest
mean in the fixed confidence or fixed budget setting~\cite{bechhofer1958sequential,bubeck2009pure,even2002pac,jamieson2014lil,kalyanakrishnan2012pac,karnin2013almost,paulson1964sequential}.
Several works have also developed hardness results for BAI~\cite{kaufmann2016complexity,mannor2004sample}.
Our algorithm and theoretical results build on this rich line of work.
In particular, our algorithm's sampling strategy is inspired by the LUCB algorithm
of~\cite{kalyanakrishnan2012pac}.
However, none of these work studies constrained BAI, where the
feasibility constraints need to be tested separately.

Our model for testing feasibility is closely related to thresholding
bandits, i.e., identify all arms whose
mean is larger than a given threshold~\cite{katz2018feasible,locatelli2016optimal,mason2020findall}.
In our work, however, it is sufficient to identify just one constraint that is larger than
the threshold to determine if it is infeasible.
Another similar line of work is on finding if there is any arm below a threshold~\cite{degenne2019pure,kano2019good,katz2020true,kaufmann2018sequential}.
\cite{hou2023} consider a variance-constrained BAI problem, where feasibility is defined via a variance threshold and is not tested separately from performance, as both are derived from the same reward distribution.
While we use some ideas from these works when testing the feasibility of an arm,
their algorithms and analysis cannot be directly applied
to our setting, since we need to determine if the most efficient way to eliminate an arm is via
its sub-optimality or infeasibility.

Perhaps the closest work to ours is~\cite{katz2019top}, who study 
identifying the best $m$ arms that satisfy the given feasibility constraints.
They assume each arm $i$ is associated with a $D$-dimensional distribution with mean
$\bm{\mu}_i$ and aim to find the top $m$ arms that maximize $\mathbf{r}^\top\bm{\mu}_i$
subject to the constraint $\bm{\mu}_i \in P$.
Here $\mathbf{r}$ is a given direction for the reward and $P$ is a subset of $\mathbb{R}^D$
denoting the feasible set.
However, this is different from our setting, since we
allow the objective and feasibility constraints to be tested separately.
A naive application of their algorithm will require testing an arm's performance and all
feasibility constraints each time when we wish to test an arm, which is sample inefficient in our setting. \citet{chen2017nearly} studies identifying the best
set in a family of feasible subsets. Our work is different from theirs since they did not consider testing for feasibility separately and they assume the feasible subsets have 
various combinatorial structures.

Our work is also related to work in the Bayesian optimization
literature that studies zeroth order constrained optimization and multi-objective optimization,
where multiple performance objectives and constraints can be evaluated separately,
similar to our setting~\cite{gardner2014bayesian,ungredda2021bayesian,eriksson2021scalable,kirschner2022tuning,hernandez2016,berkenkamp2021bayesian,kirschner2019adaptive,paria2020flexible,hernandez2016predictive}.
These problem settings are distinctly different from ours as we consider the $K$--armed version of the problem.

\section{PROBLEM SETUP}
\label{sec:setup}

For $a\in\mathbb{N}_+$, we denote $[a] = \{1,\dots,a\}$.
Given $K$ arms. 
each arm $i\in [K]$ is associated with $N+1$ distributions $\{\nuil\}_{\ell=0}^{N}$,
which are all $1$ sub-Gaussian for simplicity\footnote{The variances can be easily extended to other numbers and we will get
the same result via scaling}.
Let $\muil = \mathbb{E}_{X\sim \nuil}[X]$ be the mean of distribution $\nuil$.
We assume w.l.o.g that $\muil \in [0,1]$ for all $i,\ell$.
For any arm $i$, $\nuio$ is associated with its performance,
and
$\nuil$, for $\ell\in[N]$ is associated with
feasibility constraint $\ell$.
The set of all feasible arms $\feasible$ is defined as
\begin{align*}
    \feasible = \left\{i\in [K]\,;\; \muil < \frac{1}{2}\; \text{ for all } \ell\in [N]\right\}.
\end{align*}
If $\feasible\neq \emptyset$,  we define the optimal arm $\iopt$ as the feasible arm with the highest
expected performance.
Otherwise, we define $\iopt = K+1$ (indicates there is no feasible arm).
\begin{align}
    \iopt = \begin{cases}
        \arg\max_{i\in\feasible} \,\muio , &\text{if } \feasible \neq \emptyset \\
        K+1, &\text{if } \feasible = \emptyset
    \end{cases}
\label{eqn:ioptdefn}
\end{align}

W.l.o.g., 
we assume that $\mu_{1,0} > \mu_{2, 0} > \cdots > \muiopto > \cdots > \mu_{K,0}$.
For each arm $i$, we assume that the $\{\muil\}_{\ell\in[N]}$
are also not equal to each other and 
$\muil\neq \frac{1}{2}$ for all $i,\ell$.
Note the mean values and the ordering are unknown.

An algorithm for this problem proceeds over a sequence of
rounds, terminates, and then recommends an arm as the optimal feasible arm.
On each round $r$, based on past observations,
it chooses $(i_r, \ell_r)$ that determines which distribution, i.e., arm $i_r\in[K]$
and performance/feasibility $\ell_r\in\{0\}\cup[N]$, to sample from.
When it stops, it outputs $\widehat{a}\in[K+1]$.
If $\widehat{a}\in[K]$, then the decision-maker recommends $\widehat{a}$ as the optimal feasible
arm, and if $\widehat{a} = K+1$, then it declares that there is no feasible arm.

%

\section{LOWER BOUND}
\label{sec:lowerbound}

In this section, we will first define the complexity of the feasible BAI problem and then provide a
lower bound for any $\delta$-correct algorithm.
To motivate the ensuing discussion, consider the bandit instance in
Fig.~\ref{fig:intro},
where we have $K=5$ arms and $N=1$ feasibility constraint.
The optimal arm is $\iopt = 2$.
Since $\mu_{1,0} > \mu_{2,0}$, the only way to determine $\iopt\neq 1$ is to
verify arm $1$ is infeasible, i.e. $\mu_{1,1} > \frac{1}{2}$.
Since arm 4 is feasible $\mu_{4,1} < \frac{1}{2}$, the only way to determine that $\iopt\neq 4$ is to
verify arm 4 has worse performance than $\iopt$, i.e., $\mu_{2,0} > \mu_{4, 0}$.
Note that arms 3 and 5 are both infeasible and have lower performance than $\iopt$, thus
they can be eliminated via both ways.
Of these, we may prefer to determine $\iopt\neq3$ based on feasibility since
$\mu_{3,0}\approx\mu_{2,0}$, but $\mu_{3,1} \gg \frac{1}{2}$.
However, we may prefer to determine $\iopt\neq 5$ based on performance since
$\mu_{5,1}\approx \frac{1}{2}$, but $\mu_{5,0}\ll\mu_{2,0}$.

\begin{figure}[t]
    \centering
    \includegraphics[width = 0.45\linewidth]{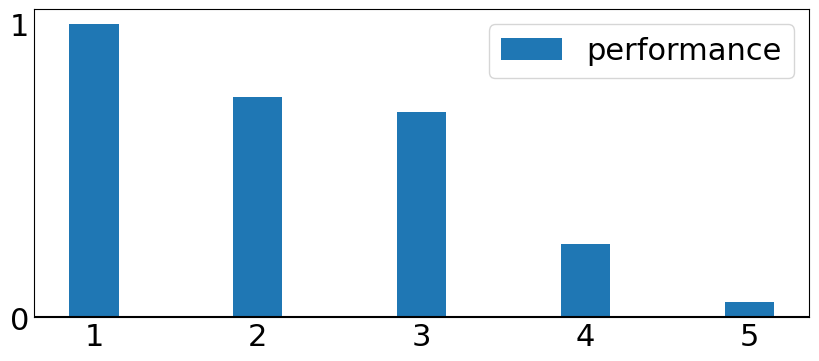}
    \includegraphics[width = 0.45\linewidth]{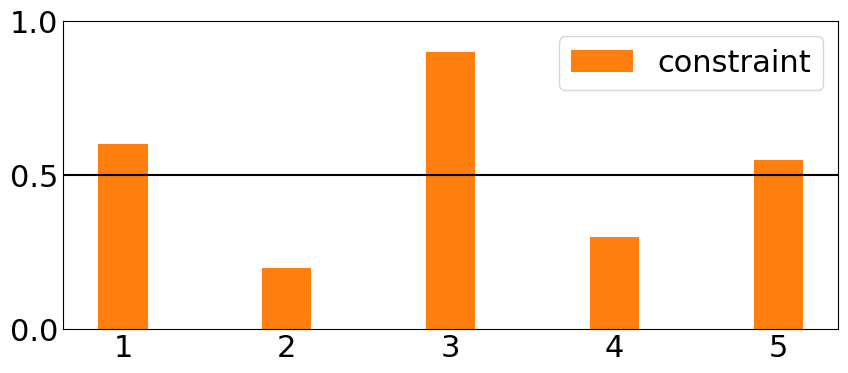}
    \caption{An example bandit instance when $K=5$ and $N=1$. The optimal feasible arm is
$\iopt=2$.}
    \label{fig:intro}
    \vspace{-0.1in}
\end{figure}

For each arm $i\in [K]$, we first define the complexity terms to determine its feasibility ($\theta_i$) and performance ($\phi_i$).
We will start with feasibility.
To verify that an arm $i$ is feasible, we have to check 
$\muil<\frac{1}{2}$ for all $\ell\in[N]$.
But to verify that an arm $i$ is infeasible, it is sufficient to check only one constraint $\ell\in[N]$ such that $\muil > \frac{1}{2}$.
If there are several such constraints, the easiest will be the 
constraint with the largest $\muil$.
Hence, we define $\theta_i$ as:
\begin{align*}
    \theta_i = 
    \begin{cases}
        (\max_{\ell \in [N]}\muil - \frac{1}{2} )^{-2}, \;\;&\text{if } i \notin \feasible, \\
        \sum_{\ell=1}^N (\mu_{i,\ell} - \frac{1}{2})^{-2}, \;\; &\text{if } i \in \feasible.
    \end{cases}
\label{eqn:thetai}\numberthis
\end{align*}
We then define the complexity term $\phi_i$ to differentiate the performance between arm $i$ and $\iopt$.
Recall that arms are arranged in decreasing order of expected performance.
For arms $i<\iopt$, i.e. those have better performance but
are infeasible, the only way to determine $i\neq \iopt$ is to test for their feasibility;
hence, we define $\phi_i=\infty$ for all $i<\iopt$.
For $i>\iopt$, we let $\phi_i = (\muiopto - \muio)^{-2}$.
Finally, for $\iopt$, we set $\phi_{\iopt}=0$ if there are no other feasible arms since it
is unnecessary to verify that its performance is better than other arms;
otherwise, we set it to the inverse squared gap between $\iopt$ and the feasible arm with the second-highest performance.
Putting this all together, we have:
\begin{align*}
\phi_i = 
\begin{cases}
     \infty,  &\text{if } i<\iopt, \\
     (\mu_{\iopt, 0} - \mu_{i, 0})^{-2}, 
        &\text{if } i>\iopt, \\
     0, &\text{if } i=\iopt, |\feasible|=1, \\
     (\mu_{\iopt, 0} - \max\limits_{j \in \feasible\backslash\{\iopt\}}\mu_{j, 0})^{-2}, 
        &\text{if } i=\iopt, |\feasible|>1, \\
\end{cases}
\numberthis\label{eqn:phii}
\end{align*}
which stands true when there is no feasible arm, i.e.
$\iopt = K+1$.
Now, we define two sets $\calI$ and $\calW$ using the complexity terms in~\eqref{eqn:thetai}
and~\eqref{eqn:phii}.
Intuitively, $\calI$ are the arms that are easier to eliminate due to infeasibility, and $\calW$ are the arms that are easier to eliminate due to their worse
performance than $\iopt$. We have,
\begin{align*}
\calI &= [\iopt-1] \cup \Big( \{\iopt+1,\dots,K\} \cap \{i \in \feasiblec:  \theta_i <
\phi_i\} \Big),\\
\calW &= \{\iopt+1,\dots,K\} \cap \Big( \feasible \cup 
 \{i \in \feasiblec:\,  \theta_i \geq \phi_i\}
\Big).
\end{align*}
Observe that $[K] = \calI \cup \calW \cup \{\iopt\}$.
This leads to the following definition of the complexity
of an arm $\calH_i$ and the complexity of the problem $\calH$ which is the sum of the arm complexities.
\begin{align*}
    \calH_i &= 
    \begin{cases}
        \theta_i, &\text{if } i \in \calI, \\
        \phi_i, &\text{if } i \in \calW, \\
        \theta_i + \phi_i, &\text{if } i = \iopt \text{ and } \iopt \le K,
    \end{cases}\\
    \calH &= \sum_{i=1}^K \calH_i. \label{eqn: H}
    \numberthis
\end{align*}
It suggests that for all suboptimal arms, we only need to verify either they are infeasible or have worse performance than $\iopt$.
For $\iopt$, we need to verify that it is both feasible and that its
mean $\muiopto$ is larger than the feasible arm with the second-highest performance.
Next, we show the lower bound of the expected sample complexity for any $\delta$-correct algorithms. 
\begin{theorem} \label{thm: lower bound}
Let $\nu$ denote a bandit instance with Gaussian observations satisfying assumptions in \S PROBLEM SETUP. Let $\delta \in (0,1)$ and $\calH$ defined in~\eqref{eqn: H}. 
Any algorithm $\calA$ that is $\delta$-correct
has a stopping time $\tau$ on $\nu$ that satisfies
\begin{align*}
    \bbE_\nu[\tau] \ge 2\calH\log\frac{1}{2.4\delta}.
\end{align*}
\end{theorem}
The lower bound in Theorem~\ref{thm: lower bound} has an expected sample complexity aligning with the complexity terms in~\eqref{eqn: H}. The main challenge in proving the lower bound 
is to construct alternate bandit instances that match the complexity terms. All proofs are in the Appendix of the extended version.

\section{ALGORITHM AND UPPER BOUND} \label{sec:alg and upper bound}

We now present our algorithm and its upper bound of the expected sample complexity. 
Our algorithm, outlined in  Algorithm~\ref{alg:ours},
proceeds over a series of epochs, indexed by $t$.
At each epoch, it will choose up to two arms $a_t, b_t\in[K]$.
For each arm, it will test the performance and/or one of its feasibility constraints.

To describe our algorithm in detail, we first define some quantities. For all $i \in
[K], \ell\in\{0\}\cup[N]$,
$N_{i,\ell}(t)$ denotes the number of times the tuple
$(i,\ell)$ has been sampled and $M_i(t)$ denotes the total number of times arm $i$ has
been tested for feasibility up to epoch $t$:
\begin{align*}
N_{i,\ell}(t) &= \sum_{s=1}^t \msi[(i,\ell) \mbox{ is sampled on epoch } s ],\\
M_i(t) &= \sum_{\ell=1}^N N_{i,\ell}(t).
\end{align*}
Next, we define the upper and lower confidence bounds for each mean value.
Let $X_{i, \ell, s}$ be the sample collected from arm $i$'s $\ell^{\rm th}$
distribution when it is sampled for the $s^{\rm th}$ time.
$\widehat{\mu}_{i,\ell}(t)$ denotes the empirical mean of the distribution up to epoch $t$.
Let $D(n,\delta)$ denote the uncertainty term, the upper and lower
confidence bounds $\overline{\mu}_{i,\ell}(t), \underline{\mu}_{i,\ell}(t)$ for $\muil$ are:
\begin{align}  \label{eqn: confidence bound}
    \widehat{\mu}_{i,\ell}(t) &= \frac{\sum_{s=1}^{N_{i,\ell}(t)}X_{i, \ell, s}}{N_{i,\ell}(t)},\;
    D(n, \delta) = \sqrt{\frac{2}{n}\log\frac{4n^4}{\delta}},\\
    \nonumber \overline{\mu}_{i,\ell}(t) &= \widehat{\mu}_{i,\ell}(t)+ D(N_{i,\ell}(t), \delta/(K(N+1)),\\
    \nonumber \underline{\mu}_{i, \ell}(t) &= \widehat{\mu}_{i,\ell}(t) - D(N_{i, \ell}(t), \delta/(K(N+1))). 
\end{align}
For all $i\in [K], \ell \in [N]$ at epoch $t$, we define $\widetilde{\mu}_{i,\ell}(t)$ as follows which will help us to determine which constraint to test when we wish to test the feasibility of arm $i$: 
\begin{align} \label{eqn: feasibility}
    \widetilde{\mu}_{i,\ell}(t) = \widehat{\mu}_{i,\ell}(t) + \sqrt{\frac{2\log M_i(t)}{N_{i,\ell}(t)}}.
\end{align}

At each epoch, Algorithm~\ref{alg:ours} samples arms that are most probable to be $\iopt$ and simultaneously eliminate suboptimal arms, which are either determined to be infeasible or have worse performance than an arm that has been determined to be feasible. 
It maintains a few sets of arms that are updated at the end of each epoch.
First, $S$ denotes the set of surviving arms,
i.e. arms that are not eliminated. Second, 
we form a focus set $P$, which contains arms in $S$ that have high performance based on the
samples collected so far, see line 18.
Next, we have  $F$ (and $I$) which are arms determined to be feasible (infeasible)
with probability at least $1-\delta$
based on the samples collected so far.
Finally, 
for each arm $i \in [K]$, we maintains set $H_i\subset[N]$
which contains the feasibility constraints of $i$ that have not
yet been determined to be feasible. 
We initialize our algorithm by pulling each arm and each distribution once. 


We first describe the main sampling rule in each epoch in lines 6--11,
and then describe lines 2--4.
On each epoch, we choose up to two arms $a_t$ and $b_t$, and evaluate their performance.
This sampling rule, which is adapted from the LUCB algorithm~\cite{kalyanakrishnan2012pac},
chooses 
$a_t$ to be the arm with the highest sample mean in $P$,
and $b_t$ to be the arm with the highest upper confidence bound on performance in $P\backslash \{a_t\}$.
Intuitively, this choice of $a_t$ and $b_t$ focuses on the arms that are most
likely to be optimal, but are hardest to distinguish.
After testing for performance, if
$a_t$ or $b_t$ have not been determined to be feasible yet, i.e. if they are not in $F$,
then they are tested for exactly one feasibility
constraint as suggested by the subroutine \sampleforsafety{}.
The case where $P$ contains only one arm $i$ (lines 3--4), indicates that its performance
has been deemed to be clearly higher than the rest of the arms in $S$.
Hence, if it is feasible, it will clearly be the optimal arm.
Thus, if it is already guaranteed to be feasible, we can stop and return this arm (line 3),
but otherwise, we should keep testing if it is feasible (line 4) until it is clear as to
whether it is feasible or not.
Finally, if $S=\emptyset$, this is because all arms have been deemed infeasible,
so we return $K+1$ (line 2).

Inspired from~\citet{kano2019good}, we design the subroutine \sampleforsafety{}, which recommends one constraint to test when we wish to test the feasibility of an arm. Each time we invoke it with an arm $i \in [K]$, it chooses the constraint
in $H_{i}$ that maximizes $\widetilde{\mu}_{i,\ell}(t)$, as in \eqref{eqn: feasibility}.
Intuitively, via this choice, we are sampling the constraint that is most likely to be larger than
the threshold $\frac{1}{2}$.
After sampling, if the lower confidence bound for the sampled constraint is larger than $\frac{1}{2}$, 
we deem it to be infeasible so we add $i$ to $I$; then $i$ will be eliminated from $S$ shortly.
otherwise, if its upper confidence bound is smaller than $\frac{1}{2}$, we have verified that the
sampled constraint is not in violation and remove it from $H_{i}$.
In the event that $H_{i}$ is empty, then none of the constraints are in violation, and hence we
add $i$ to $F$ as we have verified that $i$ is feasible.

At the end of each epoch, we update $S$ and $P$ as follows (lines 14--18).
First, we eliminate all determined infeasible arms in $I$ from $S$ (line 14).
Then, to eliminate worse-performance arms, $S$ is updated to contain arms whose upper confidence bounds for performance
are larger than the smallest lower confidence bounds in $F$.
Intuitively, if arm $i$ has been found to be feasible, and another arm $j$ has a performance
guaranteed to be lower than $i$, then we can eliminate $j$ regardless its feasibility.
Finally, we update $P$ to be the arms in $S$ whose upper confidence bounds for
performance are larger than the smallest lower confidence bounds in $S$.
Intuitively, $P$ is the set of arms in $S$ which have high performance, so our sampling in the next
the round will focus on $P$.
The benefit of forming the focus set $P$ is if $\iopt \in P$, we do not need to spend extra samples on testing the feasibility of arms in $S\backslash P$; however, if $\iopt \notin P$, all arms in $P$ would belong to $\calI$ with high probability and since we need to test their feasibility anyway, it is harmless to focus on $P$ first.
\insertAlgoOurs
\subsection{Upper bound}
We will now state our main theoretical result, which shows that
Algorithm~\ref{alg:ours} is $\delta$-correct and upper bounds its expected sample complexity on the number of epochs, which is at most four times the number of rounds.
To state this,
we define the following terms, which are related to the gaps between different performance means and feasibility means to the threshold and will be useful in Theorem~\ref{thm: upper bound},
\begin{align*}
    \Delta_{i,j} &= \mu_{i,0} - \mu_{j, 0}, \quad \quad 
    \Gamma_{i, \ell} = (\mu_{i, \ell} - \frac{1}{2})^{-2},\\
    \delta' &= \delta/(K(N+1))^{\frac{1}{4}}.
\end{align*}
\begin{theorem} \label{thm: upper bound}
Given $\delta \in (0,1)$,
for any bandit instance satisfying the assumptions in
\S PROBLEM SETUP, Algorithm~\ref{alg:ours} is $\delta$-correct.
If $\iopt\leq K$, the stopping time $\tau$ satisfies,
\begin{align}
    \nonumber \bbE[\tau] &\le \sum_{i \in \calI} 32\theta_i\log\left( \frac{91\theta_i}{\delta'} \log \frac{95\theta_i}{\delta'} \right)\\
    \nonumber &+ 292\left(\sum_{i\notin \calI} \phi_i\right) \cdot \log \frac{\sum_{i\notin \calI} \phi_i}{\delta}  \\
     \nonumber &+ \sum_{\ell=1}^N 32\Gamma_{\iopt, \ell}\log\left( \frac{91\Gamma_{\iopt, \ell}}{\delta'} \log \frac{95\Gamma_{\iopt,\ell}}{\delta'}\right)\\
     &+ G_1 + G_2 +G_3. \label{eqn: upper bound < K}
\end{align}
If $\iopt = K+1$, i.e. if no feasible arm exists, $\tau$ satisfies
\begin{align} \label{eqn: upper bound > K+1}
    \bbE[\tau] \le \sum_{i \in [K]} 32\theta_i\log\left( \frac{91\theta_i}{\delta'} \log
\frac{95\theta_i}{\delta'} \right) + G_4 +G_5.
\end{align}
The quantities $G_1,\dots,G_5$ are lower order terms which do not have a leading $\log(1/\delta)$ term or a leading $\delta$ term, which can be ignored when $\delta \to 0$, see their definition in Appendix in the extended version.
\end{theorem}
The following corollary, which is straightforward to verify by comparing the terms in
Theorem~\ref{thm: upper bound} with a leading $\log(1/\delta)$ term to $\calH$ in~\eqref{eqn: H}, shows that our algorithm matches the lower bound as
$\delta\rightarrow 0$.

\begin{corollary} \label{corollary}
    Let $\delta\in (0,1) $. Let $\calH$ be defined as in Equation~\eqref{eqn: H}. Algorithm~\ref{alg:ours} satisfies\footnote{$\widetilde\calO$ denotes up to constants and logarithmic factors. We believe the logarithmic factors could be reduced by adopting a tighter confidence bound as in~\citep{jamieson2014lil}.},
    \begin{align*}
        \limsup_{\delta \rightarrow 0} \frac{\bbE[\tau]}{\log(1/\delta)} 
        \in \widetilde\calO(\calH)
    \end{align*}
\end{corollary}
The upper bound in Theorem~\ref{thm: upper bound} has a very similar structure to the lower bound in
Theorem~\ref{thm: lower bound}. It indicates that Algorithm~\ref{alg:ours} finds the easiest way to
eliminate all non-optimal arms either by feasibility or worse performance. 
Corollary~\ref{corollary} shows that Algorithm~\ref{alg:ours} achieves the optimal sample complexity when $\delta \to 0$.

\textbf{Gap between the lower and upper bound when $\delta \nrightarrow 0$:}
It is easy to verify that the sample complexity of Algorithm~\ref{alg:ours} matches the lower bound in Theorem~\ref{thm: lower bound} when $N=1$ for any $\delta \in (0,1)$. However, when $N > 1$, Algorithm~\ref{alg:ours} can still match the lower bound for any $\delta$ on all arms except for arms in $\calI$. This mismatch is due to that Theorem~\ref{thm: lower bound} can only be achieved by an oracle who already knows the structure and the means of all arms and samples just to \emph{verify} that $\iopt$ is the optimal arm. 
Theorem~\ref{thm: lower bound} suggests that we should only sample the constraint that has the highest mean for arms in $\calI$.
However, since we assume the decision-maker does not have all the additional information, it is expected to invest in samples on other constraints (not the highest one) for \emph{exploration}. This also explains why in Thoerem~\ref{thm: upper bound}, we have the lower order terms ($G_1, \cdots, G_5$), as they account for the \emph{cost for exploration}. 
Previous works~\citep{simchowitz2017simulator, karnin2016verification} have discussed the difference between \emph{verification} and \emph{exploration} but their tools cannot be directly applied to our problem due to the definition of feasibility. \citet{kaufmann2018sequential} provides a tighter lower bound when the distributions of the arms come from an exponential family, however, it still has a large gap between their upper bound. It remains an open question to lower bound the sample complexity any algorithm need to spend on the cost for \emph{exploration}.\\
\textbf{Proof Sketch of Theorem~\ref{thm: upper bound}:}\\
Our key contribution in the upper bound is to take a multi-level decomposition on the stopping time $\tau$ based on different combinations of the two elimination criteria and different selection rules in Algorithm~\ref{alg:ours}. The decomposition also helps us to bound the cost of \emph{exploration} for the arms in $\calI$. 
The full proofs of Theorem~\ref{thm: upper bound} and Corollary~\ref{corollary} are in the
Appendix. Here we provide the proof sketch for Theorem~\ref{thm: upper bound}. 
Based on the elimination criteria,
we first decompose $\tau$ into two parts based on whether $a_t$ and $b_t$ are in
$\calI$, which is $\tau = $
\begin{align*}
    &\underbrace{\sum_{t=1}^\infty \msi[t\le\tau, a_t \text{ or }b_t \in \calI]}_{A_1} +
\underbrace{\sum_{t=1}^\infty \msi[t\le\tau, a_t \text{ and }b_t \notin \calI]}_{A_2}.
\end{align*}

For $A_1$, we decompose each arm $i \in \calI$ into two cases based on whether
$i$'s feasibility has been determined or not. When it is not determined, we upper-bound the case by the sample complexity to determine its feasibility. When determined, it is still chosen indicating it was incorrectly identified as a feasible arm, which will only happen with probability less than $\delta$ since Algorithm~\ref{alg:ours} is $\delta$-correct. We upper bound it by $\delta$ times the sample complexity to differentiate the performance of each
arm in $\calI$ with all other arms.

For $A_2$, we decompose it into two cases where $b_t \in P$ or $b_t \notin P$. $b_t\notin P$ indicates $P = \{a_t\}$ is a singleton. $P = \{\iopt\}$ can be upper bounded by the sample complexity to determine arm $\iopt$'s feasibility.
$P = \{a_t\}$ for $a_t \in \calW$ will only happen with probability
less than $\delta$, since this indicates an arm in $\calW$ is wrongly identified as better than $\iopt$ or arm $\iopt$ has been wrongly eliminated. Second, $b_t \in P$ can be decomposed to whether arm $a_t$'s and/or $b_t$'s feasibility has been determined and can be bounded similarly.
In the case that the decision-maker
chooses two arms $a_t$ and $b_t$ and the arms in $\calW$ and arm $\iopt$ are not wrongly identified
as infeasible arms, it is similar to a standard BAI problem.

\section{EXPERIMENTS}
\label{sec:experiments}

\begin{figure*}[ht!]
    \centering
    \begin{subfigure}[b]{0.32\textwidth}
        \centering
        \includegraphics[width = \textwidth]{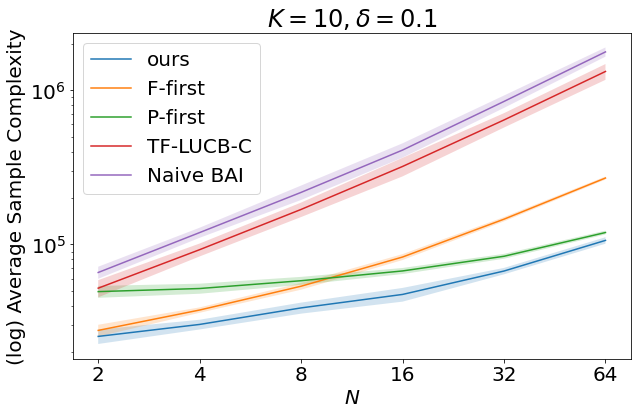}
    \end{subfigure}
    \begin{subfigure}[b]{0.32\textwidth}
        \centering
        \includegraphics[width = \textwidth]{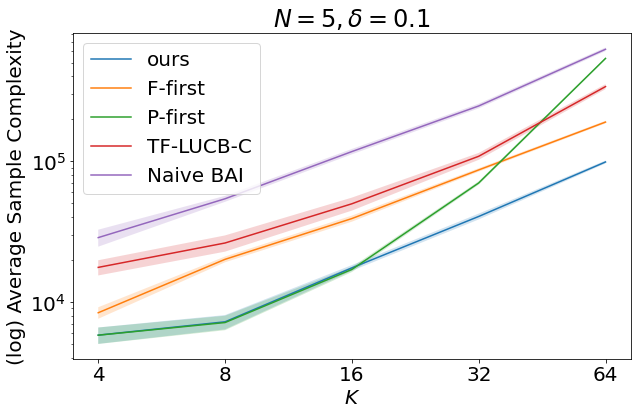}
    \end{subfigure}
    \begin{subfigure}[b]{0.32\textwidth}
        \centering
        \includegraphics[width = \textwidth]{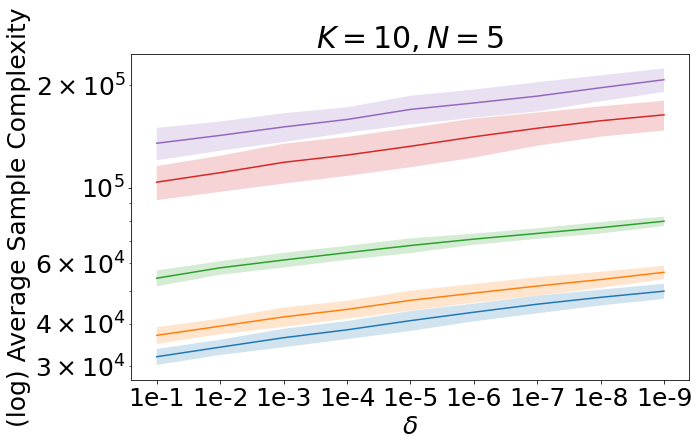}
    \end{subfigure}
    \caption{Results for Experiment 2. Results are averaged over $10$ runs and the error bars are standard deviations.
}
    \label{fig:experiment 2}
\end{figure*}

We now empirically compare our algorithm against the following four methods. For all approaches, we use the same confidence bound as in Equation~\eqref{eqn: confidence bound}.

\begin{enumerate}
    \item Feasibility-First (F-first): 
for each arm $i$, we first use the subroutine \sampleforsafety{}, which samples one constraint at each round, to determine whether it is feasible. 
Then, we apply LUCB algorithm~\cite{kalyanakrishnan2012pac} to find the best arm in the feasible set.

    \item Performance-First (P-first): We first run LUCB algorithm~\cite{kalyanakrishnan2012pac} on all
arms' performance and then test the best arm's feasibility using the subroutine \sampleforsafety{}. If the best arm is feasible, it stops, otherwise it
eliminates this arm and repeats the same procedure on the remaining arms. 
    \item TF-LUCB-C~\citep{katz2019top}: This procedure was developed for feasible BAI, but assume the performance and feasibility constraints can be sampled together.
We implement the most natural adaptation of it to our setting: pull all
performance and feasibility distributions together using $N+1$ samples.
    \item Naive BAI (Naive): It adapts from the Racing algorithm~\citep{even2002pac}, which is optimal for BAI.  On each round, it eliminates the arms that are considered
infeasible or have lower performance than arms that are already considered feasible. It
repeatedly tests the performance and all feasibility constraints of all survival arms. It stops when
only one feasible arm is left or no arm is considered feasible. 
\end{enumerate}

\textbf{Experiment 1:}
We conduct a series of experiments to demonstrate our algorithm can naturally adapt to
the difficulty of the problem.
We consider a setting where $K=5$ and $N=3$ and all distributions to be Gaussian with variance 1. We set $\delta = 0.1$ and change the means of these distributions to construct three types of problem instances. 
The settings are as follows:
\begin{table}[t]
\centering
\begin{tabular}{|l|l|l|l|l|l|}
\hline
    & ours & F-first & P-first & TF-LUCB-C & Naive \\ \hline
a & 1.0  & 0.58    & 3.09    & 1.77      & 1.81      \\ \hline
b & 1.0     & 3.13       &   0.84      &  1.98         &     4.75      \\ \hline
c & 1.0     &   4.00      &   3.47      &    2.78      &     4.06      \\ \hline
\end{tabular}
\caption{Results for Experiment 1: number of samples required relative to Algrithm~\ref{alg:ours}, averaged over 10 runs. The standard deviations are reported in the Appendix.}
\label{tbl:experiment 1}
\end{table}

\begin{enumerate}
\item
$\iopt=5$ and all other arms have a higher reward but are infeasible.
The performance means of all arms are linearly spaced in $[0, 1]$ and the means of the feasibility
distributions are $[0.75, 0.25, 0.25]$ for infeasible arms and $[0.25]^3$ for $\iopt$.
In this setting, we expect $F$-first to perform well and $P$-first to perform poorly because $P$-first will spend unnecessary samples on comparing the performances between each arm while the $F$-first algorithm directly eliminates the non-optimal arms by infeasibility.
\item $\iopt = 1$ and all other arms are feasible but have a lower reward. For all arms, the performance means are linearly spaced in $[1,0]$ and the feasibility constraints are $[0.4]^3$. 
\item We consider a problem between the two extremes and contains multiple types of suboptimal arms. We set $\iopt = 2$ and only arms $2$ and $5$ are feasible. The performance means are $[1, 0.9, 0.5, 0.25, 0]^T$. The feasibility constraints are: $[0.65, 0.4, 0.4]^T$ for arms $1,3,4$; $[0.3, 0.4, 0.4]^T$ for arm $2$; $[0.45, 0.45, 0.45]^T$ for arm $5$. 
\end{enumerate}

\textbf{Synopsis for Experiment 1:} 
Recall the key challenges, we expect F-first to perform the best on instance (a) but poorly on instance (b); and the opposite for P-first. 
The results in Table~\ref{tbl:experiment 1} match our expectations.
Note that Algorithm~\ref{alg:ours} only performs slightly worse than the best algorithm on instances (a) and (b). Moreover, on instance (c) which contains multiple types of suboptimal arms, Algorithm~\ref{alg:ours} outperforms all other methods, showing that it can adapt to the difficulty of the problem. 

\textbf{Experiment 2:}
We compare the five methods when the number of arms, constraints, and $\delta$ change. We assume the distributions are Gaussian with variance 1. 
First, we set $K=10$, $\delta = 0.1$, and increase $N$ from $2$ to $64$. The performance means of all arms are linearly spaced in $[2,0]$. Arms 4--7 are feasible arms and have feasibility constraints means to be $[0.25]^N$. The other arms are infeasible arms and have feasibility constraints to be $[0.75]\times [0.25]^{N-1}$.
Second, we set $N=5$, $\delta = 0.1$, and increase $K$ from $4$ to $64$. The performance means of all arms are linearly spaced in $[10,0]$.  Arms indexes from $\lfloor\frac{K}{3}\rfloor$ to $\lfloor\frac{2K}{3}\rfloor$ are feasible arms whose feasibility constraints means are set to $[0.25]^N$. The other arms are infeasible arms and have feasibility constraint means to be $[0.75]\times [0.25]^{N-1}$.
Third, we set $K = 10$, $N = 5$, and decrease $\delta$ from $10^{-1}$ to $10^{-9}$. The mean values are set to the same when we change $K$. 
The results are given in
Fig.~\ref{fig:experiment 2}.
We see that our method achieves
the best performance when compared to all other methods. 

\begin{figure}[t]
    \centering
    \includegraphics[width = 0.8\columnwidth]{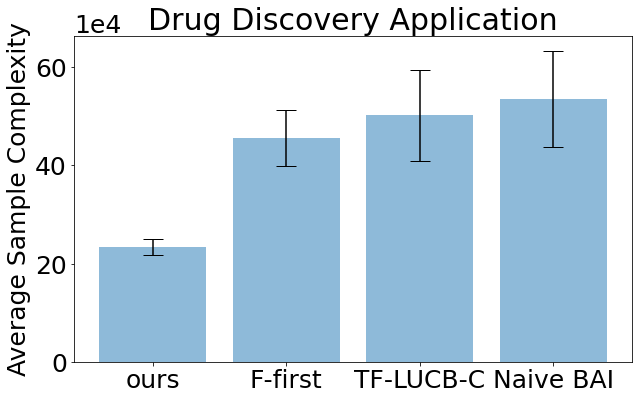}
    \caption{Drug discovery application: we exclude P-first due to poor performance. Results are averaged over 10 runs and the error bars are the standard deviations.}
    \label{fig:drug discovery}
    \vspace{-0.2in}
\end{figure}

\textbf{Drug Discovery:} We investigate finding the most effective drug satisfying different feasibility constraints. The data comes from Tables 2 \& 3 in~\citet{genovese2013efficacy}. Each arm corresponds to a dosage level (in mg): [25, 75, 150, 300, Placebo] and has $3$ attributes: probability of being effective $\mu_{i,0}$, probability of any adverse event $\mu_{i,1}$, and probability of causing an infection or infestation $\mu_{i,2}$. The threshold for constraint 1 is 0.5 and for constraint 2 is 0.25. We manually increase 0.25 to all rewards received from constraint 2 to make the thresholds equal. The means for different arms are $(.34, .519, .259)^T$, $(.469, .612, .184)^T$, $(.465, .465, .209)^T$, $(.537, .61, .293)^T$, and $(.36, .58, .16)^T$. We use Gaussian distribution with variance $1$. Fig.~\ref{fig:drug discovery} shows that our algorithm outperforms all other methods.

In all the experiments, all methods can find $\iopt$ correctly. Moreover, Algorithm~\ref{alg:ours}, P-first, and F-first, which all use \sampleforsafety{}, generally perform better than the other two methods which sample all constraints simultaneously, demonstrating that we should test each constraint individually. Since different confidence bounds will affect the performance of the algorithm~\citep{jamieson2014lil}, for a fair comparison, we use \eqref{eqn: confidence bound} on all methods. In the appendix, To see the influence of the confidence bounds on the methods, we compare Algorithm~\ref{alg:ours} and TF-LUCB-C with two different confidence bounds to see the influence of different confidence bounds. It shows that our algorithm outperforms TF-LUCB-C in both bounds and even when only TF-LUCB-C uses a tighter confidence bound.

\section{CONCLUSION}
In this work, we consider a new formalism for best arm identification (BAI) in the fixed confidence
setting when the optimal arm needs to satisfy multiple given feasibility constraints, and the
feasibility constraints can be tested separately from the performance.
We quantify the complexity of this problem by considering the relative difficulty of deciding if
an arm, that is not the optimal feasible arm, is easier to eliminate via its performance or its
feasibility.
We propose a $\delta$-correct algorithm and
upper bound its expected sample complexity, showing that it is asymptotically optimal compared to
the lower bound. Finally, we conduct empirical experiments showing our algorithm outperforms
other natural baselines for this problem.

\bibliography{example_paper}

\newpage
\onecolumn
\section*{Appendix}


\section{Notation Table}
Here we summarize all the notations that appear in the paper and the Appendix. 
\begin{table}[H]\caption{Notation Table}
\begin{center}
\begin{tabular}{r c p{10cm} }
\toprule
$K$  && number of arms \\
$N$ && number of feasibility constraints \\
$\delta$ && failure rate\\
$\feasible$ && $\feasible = \left\{i\in [K]\,;\; \muil < \frac{1}{2}\; \text{ for all } \ell\in [N]\right\}$, set of feasible arms \\
$\iopt$ && $\iopt = \begin{cases}
        \arg\max_{i\in\feasible} \,\muio , &\text{if } \feasible \neq \emptyset \\
        K+1, &\text{if } \feasible = \emptyset
    \end{cases}$, the optimal arm\\
$\mu_{i, 0}$ && expected reward of arm $i$ on performance (unknown) \\
$\mu_{i, \ell}$ && expected reward of arm $i$ on feasibility constraint $\ell$ (unknown)\\

$\theta_i$ && complexity of feasibility for arm $i$, see Equation~\eqref{eqn:thetai}\\
$\phi_i$ && complexity of performance for arm $i$, see Equation~\eqref{eqn:phii} \\
$\calI$ && $\calI = \{1,\dots,\iopt-1\} \cup \Big( \{\iopt+1,\dots,K\} \cap \{i \in \feasiblec:  \theta_i <
\phi_i\} \Big)$, set of arms that should be eliminated by feasibility\\
$\calW$ && $\calW = \{\iopt+1,\dots,K\} \cap \Big( \feasible \cup 
 \{i \in \feasiblec:\,  \theta_i \geq \phi_i\}
\Big)$, set of arms that should be eliminated by performance \\
$H$ && complexity of the problem, see Equation~\eqref{eqn: H}\\

$t$ && number of epoch of Algorithm~\ref{alg:ours}\\
$\widehat{\mu}_{i, \ell}(t)$ && empirical mean of the rewards of arm $i$'s $\ell^{\rm th}$ distribution at epoch $t$, see Equation~\eqref{eqn: confidence bound} \\
$\widehat{\mu}_{i,\ell,s}$ && empirical mean of the rewards of  arm $i$'s $\ell^{\rm th}$ distribution when $(i,\ell)$ has been sampled $s$ times\\
$N_{i,\ell}(t)$ && $N_{i,\ell}(t) = \sum_{s=1}^t \msi[(i,\ell) \mbox{ is sampled on epoch } s ]$, number of times arm $i$'s performance $(\ell=0)$ or feasibility constraint $\ell$ has been sampled up to at epoch $t$ \\
$M_i(t)$ && $M_i(t) = \sum_{\ell=1}^N N_{i,\ell}(t)$, total number of times we have sampled arm $i$'s feasibility constraints up to epoch $t$\\
$D(n,\delta)$ && $D(n,\delta) = \sqrt{\frac{2}{n}\log\frac{4n^4}{\delta}}$, confidence bound\\
$\overline{\mu}_{i,\ell}(t)$ && upper confidence bound of arm $i$'s $\ell^{\rm th}$ distribution at epoch $t$ , see Equation~\eqref{eqn: confidence bound}\\
$\underline{\mu}_{i,\ell}(t)$ && lower confidence bound of arm $i$'s $\ell^{\rm th}$ distribution at epoch $t$, see Equation~\eqref{eqn: confidence bound}\\
$\widetilde{\mu}_{i,\ell}(t)$ && sampling rule sample feasibility constraint $\ell$ of arm $i$ at epoch $t$, see Equation~\eqref{eqn: feasibility}  \\

$\tau$ && epoch that the algorithm stops  \\
$\Delta_{i,j}$ && $\Delta_{i,j} = (\mu_{i,0} - \mu_{j, 0})/4$, a (scaled) gap between the performance of the $i^{\rm th}$ and $j^{\rm th}$ arms.\\
$\Gamma_{i,\ell}$ && $\Gamma_{i,\ell} = (\mu_{i,\ell} - 1/2)^{-2}$, inverse squared gap between the mean value and the threshold 
1/2 of arm $i$'s feasibility constraint $\ell$ \\
$\delta'$ && $(\delta/(K(N+1)))^{\frac{1}{4}}$\\
$\Delta_{ij}$ && $\Delta_{ij} = \mu_{i,0} - \mu_{j, 0}$, gap between the performance of $i^{\rm th}$ and $j^{\rm th}$ arms.\\

\bottomrule
\end{tabular}
\end{center}
\end{table}

\newpage
Next, we tabulate some notations used in Section~\ref{subsec:feasibility}. The square parenthesis indicates what happens when an arm's feasibility has been sampled up to $s$ times and not $s$ epochs.
For example, $N_{i,\ell}[s]$ denotes the number of times feasibility constraint $\ell$ of arm $i$ has been sampled when arm $i$'s feasibility has been sampled up to $s$ times. Its relationship to $N_{i,\ell}(t)$ can be written as
\begin{align*}
    N_{i,\ell}(t) =  N_{i,\ell}[M_i(t)]. 
\end{align*}

\begin{table}[H]\caption{Notation Table for Section~\ref{subsec:feasibility}}
\begin{center}
\begin{tabular}{r c p{10cm} }
\toprule
$\tau_{F,i}$  && number of samples needed to determine arm $i$'s feasibility \\
$h_i[s]$ && suggested constraint to sample when $\sampleforsafety$ for arm $i$ is invoked $s$ times\\
$H_i[s]$ && set of undetermined constraints of arm $i$ when $\sampleforsafety$ for arm $i$ is invoked $s$ times\\
$\widehat{\mu}_{i,\ell}[s]$ && empirical mean of arm $i$'s feasibility constraint $\ell$ when $\sampleforsafety$ for arm $i$ is invoked $s$ times \\
$N_{i,\ell}[s]$ && $N_{i,\ell}[s] = \sum_{l=1}^s \msi[h_i[l] = \ell]$, the number of times constraint $\ell$ has been sampled when $\sampleforsafety$ for arm $i$ is invoked $s$ times\\
$\widetilde{\mu}_{i,\ell}[s]$ && $\widetilde{\mu}_{i,\ell}[s] = \widehat{\mu}_{i,\ell}[s] + \sqrt{\frac{2\log s}{N_{i,\ell}[s]}}$\\
$\widetilde{\mu}_i^*[s]$ &&$\widetilde{\mu}_i^*[s] = \max_{\ell \in H_i[s]} \widetilde{\mu}_{i,\ell}[s]$\\
$\overline{\mu}_{i,\ell}[s]$ && $\overline{\mu}_{i,\ell}[s] = \widehat{\mu}_{i,\ell}[s] + \sqrt{\frac{2\log(4K(N+1)N^4_{i,\ell}[s]/\delta)}{N_{i,\ell}[s]}}$ \\
$T_i$ && $T_i = N\max_{\ell \in [K]}\lfloor n_{i,\ell} + 2\rfloor$\\
$n_{i,\ell}$ && $n_{i,\ell} = \frac{8}{(\gamma_i(\ell,1/2)/2)^2} \log \frac{8^{\frac{3}{2}} (K(N+1)/\delta)^{\frac{1}{4}} }{(\gamma_i(\ell,1/2)/2)^2} \log \frac{(8^{\frac{3}{2}}+1)(K(N+1)/\delta)^{\frac{1}{4}} }{(\gamma_i(\ell,1/2)/2)^2}.$\\
$\gamma_i(\ell,m)$&& $\gamma_i(\ell,m) = |\mu_{i,\ell} - \mu_{i,m}|$ \\
$\gamma_i(\ell, 1/2) $&& $\gamma_i(\ell, 1/2) = |\mu_{i, \ell} - 1/2|$\\
\bottomrule
\end{tabular}
\end{center}
\end{table}

\section{Lower order terms in Theorem 2}
\label{sec:lower order terms}
\begin{align*}
    G_1 &= \sum_{\ell=1}^N 8\Gamma_{\iopt, \ell} + 16, \quad G_2 = \sum_{i \in \calI} g_i, \quad G_4 = \sum_{i \in [K]}  g_i,\\
     G_3 &= \delta \Bigg(4\sum_{i \in \calI} \sum_{j\neq i}q_{i,j} + \sum_{i \in \calW} \sum_{\ell=1}^N \Big(n_{i,\ell} + 8\Gamma_{i,\ell}\Big)\\
     &\hspace{0.2in}+ 2\sum_{i,j\notin \calI} q_{i,j} \Bigg),  \quad G_5 = 4\delta\sum_{i \in [K]} \sum_{j \neq i} q_{i,j},\\
    g_i &= 8\Gamma_{i,1} + \sum_{\ell =1}^N 8\Gamma_{i,\ell}N(K(N+1))^{-\frac{1}{8}} \\
    &+ \frac{32N - 22+ 16\log\frac{32}{\left(\max\limits_{m\in[N]} \mu_{i,m}-\max\limits_{n\in[N]\backslash\{m\}} \mu_{i,n} \right)^2 }}{\left(\max\limits_{m\in[N]} \mu_{i,m}-\max\limits_{n\in[N]\backslash\{m\}} \mu_{i,n} \right)^2} \\
    &+\sum_{\ell =2}^N \left( 
    \frac{8\log N\max\limits_{\ell\in[N]} \lfloor n_{i,\ell}+2\rfloor}{\left(\max\limits_{m\in[N]}\mu_{i,m}-\mu_{i,\ell}\right)^2}  + \delta'^4 \left(n_{i,\ell} + 8\Gamma_{i,\ell} \right) \right) \\
    q_{i,j} &= \frac{128}{\Delta_{i,j}^2}\log\left(\frac{368}{\Delta_{i,j}^2\delta'} \log \frac{384}{\Delta_{i,j}^2\delta'} \right) + \delta'^4,\\
    n_{i,\ell} &= 32\Gamma_{i,\ell} \log \left(\frac{91\Gamma_{i,\ell}}{\delta'}\log\frac{95\Gamma_{i,\ell}}{\delta'} \right).
\end{align*}

\section{Proof of lower bound (Theorem~\ref{thm: lower bound})}

In this section, we will state the proof of our lower bound.

\paragraph{Preliminaries:}
We will assume Gaussian rewards for this lower bound. 
 That is,  when the decision-maker chooses the tuple $(i,\ell)$, i.e. arm $i$ and performance (if $\ell=0$) or feasibility constraint $\ell$ (if $\ell > 0$), it will receive a stochastic reward drawn from the distribution $\nu_{i,\ell} = \calN(\mu_{i,\ell}, 1)$.

The main challenge in the proof of this lower bound (with an asymptotically matching upper bound) is in defining the complexity terms in Section~\ref{sec:lowerbound}.
Our proof essentially considers different cases, and then
applies the change-of-measure lemma from~\citet{kaufmann2016complexity}. 
Although this lemma was presented for the case where each arm only contains one distribution, 
it easily generalizes to our setting where we have $K(N+1)$ distributions, as we assume no dependence between the distributions. 
For completeness, we have stated this result in Lemma~\ref{appendix: lemma: change of distribution} in Appendix~\ref{sec:usefulresults}.



\paragraph{Strategy:}
In this proof, we first consider the case when $\iopt \le K$ and handle the case $\iopt = K+1$ at the end. 
Our plan of attack is as follows. We will individually consider each arm $i\in[K]$ and 
construct 1--2 alternate instances by  perturbing the performance distribution and/or the feasibility constraint distributions. We will then apply Lemma~\ref{lem:changeofmeasure} to state a lower bound on the number of times we must pull an arm $i$.
We will consider three cases for $i$ and handle them separately:
\emph{A}: $i\in\calI$,
\emph{B}: $i\in\calW$,
and \emph{C}: $i=\iopt$.


\paragraph{\emph{Case (A):} $i\in \calI$.} Recall that $\calI = \{1, \cdots, \iopt -1\}\cup(\{\iopt +1, \cdots, K\} \cap \{i\in \feasiblec:\theta_i<\phi_i\})$, we will first consider the arms
in $\{1, \cdots, \iopt -1\}$ and then consider the arms in $(\{\iopt +1, \cdots, K\} \cap \{i\in \feasiblec:\theta_i<\phi_i\})$ as they require different constructions of the alternate instance.

\emph{Case A.1:}
First, fix an arm $i \in \{1, \cdots, \iopt -1\}$. Note that arm $i$ is an infeasible arm with better performance than arm $\iopt$ in the original instance $\nu$. Now suppose in $\nu$, arm $i$ has $m$ feasibility constraint distributions, denoted as $\ell_{[1]}, \ell_{[2]}, \cdots, \ell_{[m]}$, with mean values greater than $1/2$. We will consider the alternate instance $\nu'$ where arm $i$ becomes feasible by changing the $m$ feasibility constraints of arm $i$ to be less than $1/2$ and keeping all other distributions unchanged. In other words, letting $\alpha > 0$,
\begin{align*}
    &\forall n \in [m], \nu'_{i,\ell_{[n]}} = \calN(1/2-\alpha, 1),\\
    &\forall \ell \in \{0\}\cup[N]\backslash \{\ell_{[1]}, \cdots, \ell_{[m]}\}, \nu'_{i,\ell} = \nu_{i,\ell}, \hspace{0.3in}
    \forall j \neq i, \ell \in \{0\}\cup [N], \nu'_{j, \ell} = \nu_{j,\ell}.
\end{align*}
Recall that an algorithm will output a recommendation arm $\widehat{a}$ as the optimal arm when it stops. Let $\calA = \{\widehat{a} = \iopt\}$ (where here $\iopt$ denotes the optimal arm in $\nu$). Note that in $\nu'$, arm $i$ is feasible and has better performance than $\iopt$. It follows that arm $i$ is the optimal arm under instance $\nu'$, so for any $\delta$-correct algorithm, $\bbP_\nu(\calA) \ge 1-\delta$ and $\bbP_{\nu'}(\calA) <\delta$. In addition, using the properties of Gaussian distributions, we have that
\begin{align*}
    \forall n \in [m], \rmK\rmL(\nu_{i, \ell_{[n]}}, \nu'_{i, \ell_{[n]}}) = \frac{(1/2-\alpha-\mu_{i, \ell_{[n]}})^2}{2}.
\end{align*}
Lower bounding $d(\bbP_\nu(\calA), \bbP_{\nu'}(\calA)) \ge \log \frac{1}{2.4\delta}$ as in~\citet{kaufmann2016complexity} and letting $\alpha \to 0$, we conclude that
\begin{align*}
    \sum_{n=1}^m \bbE_{\nu} [N_{i, \ell_{[n]}}(\tau)] \cdot \frac{(1/2-\mu_{i, \ell_{[n]}})^2}{2} \ge \log \frac{1}{2.4\delta}.
\end{align*}
Now we choose $k \in [m]$ that maximizes $\frac{(1/2-\mu_{i, \ell_{[k]}})^2}{2}$ since we want to lower bound $\sum_{n=1}^m \bbE_{\nu}[N_{i, \ell_{[n]}}(\tau)]$, we have
\begin{align*}
    \sum_{n=1}^m \bbE_\nu[N_{i, \ell_{[n]}}(\tau)] \cdot \left(\max_{k \in [m]} \frac{(1/2-\mu_{i, \ell_{[k]}})^2}{2} \right) \ge \sum_{n=1}^m \bbE_{\nu} [N_{i, \ell_{[n]}}(\tau)] \cdot \frac{(1/2-\mu_{i, \ell_{[n]}})^2}{2}. 
\end{align*}
Putting these together, we have
\begin{align*}
    \sum_{n=1}^m \bbE_{\nu}[N_{i, \ell_{[n]}}(\tau)] \ge \frac{2}{\max_{k\in[m]}(1/2-\mu_{i, \ell_{[n]}})^2}\log\frac{1}{2.4\delta}.
\end{align*}
Recall the definition of $\theta_i = \frac{1}{(\max_{\ell\in[N]}\mu_{i, \ell-1/2})^2}$, see Equation~\eqref{eqn:thetai}, when $i \notin \calF$, since $i$ here is an infeasible arm in $\nu$ and $\{\ell_{[1]}, \cdots, \ell_{[m]}\}$  are the feasibility constraints of arm $i$ that are greater than $1/2$, it is easy to see that
\begin{align*}
    \sum_{n=1}^m \bbE_\nu[N_{i, \ell_{[n]}}(\tau)] \ge 2\theta_i\log\frac{1}{2.4\delta}.
\end{align*}

\emph{Case A.2:}
Second, fix an arm $i \in (\{\iopt +1, \cdots, K\} \cap \{i\in \feasiblec:\theta_i<\phi_i\})$. Note that arm $i$ is an infeasible arm with lower performance than arm $\iopt$ in the original instance $\nu$. In particular, it is easier to be eliminated by feasibility constraints since $\theta_i < \phi_i$. We follow the notation of $\{\ell_{[1]}, \cdots, \ell_{[m]}\}$ in the previous case as the feasibility constraints of arm $i$ that are greater than $1/2$. 
Here simply changing arm $i$ to be feasible is not enough to make it optimal since its performance would still be lower than arm $\iopt$. 
So we consider the alternate instance $\nu'$ where arm $i$ becomes feasible by changing the $m$ feasibility constraints of arm $i$ to be less than $1/2$ and additionally increasing the mean of arm $i$'s performance distribution to be larger than $\mu_{\iopt, 0}$, and keeping all other distributions unchanged. To be more specific, let $\alpha > 0$,
\begin{align*}
    \nu'_{i, 0} = \calN(\mu_{\iopt, 0} + \alpha, 1),
    \hspace{0.3in}
    &\forall n \in [m], \nu'_{i, \ell_{[n]}} = \calN(1/2-\alpha, 1), \hspace{0.3in}\\
    \forall \ell \in \{0\}\cup[N]\backslash \{\ell_{[1]}, \cdots, \ell_{[m]}\}, \nu'_{i,\ell} = \nu_{i,\ell}, \hspace{0.3in}
    &\forall j \neq i, \ell \in \{0\}\cup [N], \nu'_{j, \ell} = \nu_{j,\ell}.
\end{align*}
Now arm $i$ is the optimal arm under instance $\nu'$. Use the same event $\calA\{\widehat{a} = \iopt\}$, we have
\begin{align*}
    \forall n \in [m], \rmK\rmL(\nu_{i, \ell_{[n]}}, \nu'_{i, \ell_{[n]}}) &= \frac{(1/2-\alpha-\mu_{i, \ell_{[n]}})^2}{2}. \\
    \rmK\rmL(\nu_{i, 0}, \nu'_{i, 0}) &= \frac{ (\mu_{\iopt,0}+\alpha - \mu_{i, 0})^2 } {2}.
\end{align*}
Letting $\alpha \to 0$, we conclude that
\begin{align*}
    \bbE_{\nu}[N_{i, 0}(\tau)]\cdot\frac{(\mu_{\iopt,0}-\mu_{i,0})^2}{2} + \sum_{n=1}^m\bbE_\nu[N_{i, \ell_{[n]}}(\tau)] \cdot \frac{(1/2-\mu_{i, \ell_{[n]}})^2}{2}  \ge \log \frac{1}{2.4\delta}.
\end{align*}
Recall the definition of $\theta_i$ and $\phi_i = (\mu_{\iopt,0}-\mu_{i,0})^{-2}$ for arm $i > \iopt$, see Equation~\eqref{eqn:thetai} and \eqref{eqn:phii}, we have
\begin{align*}
    \frac{1}{2\phi_i}\cdot \bbE_{\nu}[N_{i,0}(\tau)] + \frac{1}{2\theta_i}\cdot \sum_{n=1}^m\bbE_{\nu}[N_{i, \ell_{[n]}}(\tau)] \ge \log\frac{1}{2.4\delta}.
\end{align*}
Since arm $i \in(\{\iopt +1, \cdots, K\} \cap \{i\in \feasiblec:\theta_i<\phi_i\}) $, we have
\begin{align*}
    \bbE_\nu[N_{i,0}(\tau)] + \sum_{n=1}^m\bbE_{\nu}[N_{i, \ell_{[n]}}(\tau)] \ge2\theta_i \log \frac{1}{2.4\delta}.
\end{align*}
Putting the two cases together, we have that for $i \in \calI$,
\begin{align*}
    \sum_{\ell=0}^N \bbE_\nu[N_{i, \ell}(\tau)] \ge 2\theta_i\log\frac{1}{2.4\delta}.
\end{align*}

\paragraph{\emph{Case (B):} $i\in \calW$.} 
Recall that $\calW = \{\iopt+1, \cdots, K\} \cap \left(\feasible \cup \{i \in \feasiblec: \theta_i \ge \phi_i\} \right)$, we will first consider arms in $\{\iopt+1, \cdots, K\}\cap\feasible$ and the consider the arms in $\{\iopt +1, \cdots, K\}\cap\{i\in\feasiblec: \theta_i \ge \phi_i\}$ as they require different alternate instances.

\emph{Case B.1:}
First, fix an arm $i \in \{\iopt + 1, \cdots, K\}\cap \feasible$. Note that arm $i$ is a feasible arm with lower performance than arm $\iopt$ in the original instance $\nu$. We will consider the alternate instance $\nu'$ where arm $i$ is increased to be larger than $\mu_{\iopt, 0}$ and keep all other distributions unchanged. Let $\alpha > 0$,
\begin{align*}
    &\nu'_{i, 0} = \calN(\mu_{\iopt, 0} + \alpha, 1),\\ 
    &\forall \ell \in [N], \nu'_{i, \ell} = \nu_{i, \ell},
    \hspace{0.3in}
    \forall j \neq i, \ell \in \{0\}\cup[N], \nu'_{j, \ell} = \nu_{j, \ell}.
\end{align*}
Now arm $i$ is the optimal arm under the instance $\nu'$. Under the event $\calA$ and let $\alpha \to 0$, we have
\begin{align*}
    \bbE_{\nu}[N_{i,0}(\tau)] \cdot\frac{(\mu_{\iopt,0}-\mu_{i,0})^2}{2}\ge\log\frac{1}{2.4\delta}.
\end{align*}
Recall the definition of $\phi_i$ in~\eqref{eqn:phii}, we have
\begin{align*}
    \bbE_{\nu}[N_{i,0}(\tau)] \ge 2\phi_i \log\frac{1}{2.4\delta}.
\end{align*}

\emph{Case B.2:} 
Second, fix an arm $i \in \{\iopt +1, \cdots, K\}\cap\{i\in\feasiblec: \theta_i \ge \phi_i\}$. This is similar to \emph{Case A.2} so we can construct the same alternate instance and get the same result as follows,
\begin{align*}
    \frac{1}{2\phi_i}\bbE_\nu[N_{i,0}(\tau)] + \frac{1}{2\theta_i}\sum_{n=1}^m\bbE_{\nu}[N_{i, \ell_{[n]}}(\tau)] \ge \log\frac{1}{2.4\delta}.
\end{align*}
However, note that $\theta_i \ge \phi_i$ here, so we have
\begin{align*}
    \bbE_\nu[N_{i,0}(\tau)] + \sum_{n=1}^m \bbE_\nu[N_{i,\ell_{[n]}}(\tau)] \ge 2\phi_i\log\frac{1}{2.4\delta}.
\end{align*}
Putting the two cases together, we have for all arm $i \in \calW$
\begin{align*}
    \sum_{\ell=0}^N \bbE_\nu[N_{i,\ell}(\tau)] \ge 2\phi_i\log\frac{1}{2.4\delta}.
\end{align*}

\paragraph{\emph{Case (C):} $i=\iopt$.} 
We will construct two alternate instances since we can make arm $i$ not optimal by either perturbing its feasibility or performance.

\emph{Case C.1:}
First, we consider the alternate instance that arm $i$ becomes infeasible. Note that we can construct $N$ such kind of alternate instances because arm $i$ can be infeasible by changing only one of its $N$ feasibility constraints. So for each $n \in [N]$, we construct the alternate instance $\nu'$ where the $n^{\rm th}$ feasibility constraint of arm $i$ is increased to be larger than $1/2$ and keeping all other distributions unchanged. Let $\alpha > 0$, we have
\begin{align*}
    &\nu'_{i, n} = \calN (1/2+\alpha, 1),\\
    &\forall \ell \in \{0\}\cup[N]\backslash\{n\}, \nu'_{i,\ell} = \nu_{i,\ell},
    \hspace{0.3in}
    \forall j \neq i, \ell\in\{0\}\cup[N], \nu'_{j,\ell} = \nu_{j, \ell}.
\end{align*}
Now arm $i$ is not the optimal arm under the instance $\nu'$ since it is infeasible. Under the event $\calA$ and let $\alpha \to 0$, we have
\begin{align*}
    \bbE_{\nu}[N_{i, n}(\tau)] \cdot \frac{(1/2-\mu_{i,n})^2}{2} \ge \log\frac{1}{2.4\delta}.
\end{align*}
Thus for all $n \in [N]$, we have
\begin{align*}
    \bbE_{\nu}[N_{i,n}(\tau)] \ge \frac{2}{(1/2-\mu_{i,n})^2}\log\frac{1}{2.4\delta}.
\end{align*}
By taking a summation, we have
\begin{align*}
    \sum_{\ell=1}^N \bbE_\nu[N_{i,\ell}(\tau)] \ge 2 \left(\sum_{\ell=1}^N\frac{1}{(1/2-\mu_{i,\ell})^2}\right)\log\frac{1}{2.4\delta}.
\end{align*}
Recall the definition of $\theta_i = \sum_{\ell=1}^N \frac{1}{(\mu_{i,\ell}-1/2)^2}$ for $\iopt$ in Equation~\eqref{eqn:thetai}, we have
\begin{align*}
    \sum_{\ell=1}^N \bbE_\nu[N_{i,\ell}(\tau)] \ge 2\theta_i\log\frac{1}{2.4\delta}.
\end{align*}

\emph{Case C.2:}
Second, we consider the alternate instance that arm $i$ has lower performance than arm $k = \argmax_{j \in \feasible\backslash\{\iopt\}} \mu_{j,0}$, which is the feasible arm with the second highest performance mean in the original instance $\nu$. We construct the alternate instance $\nu'$ by decreasing the performance mean of arm $i$ to be lower than the performance mean of arm $k$ and keeping all other distributions unchanged. Let $\alpha > 0$, we have
\begin{align*}
    &\nu'_{i,0} = \calN(\mu_{k,0}-\alpha),\\
    &\forall \ell \in [N], \nu'_{i,\ell} = \nu_{i, \ell},
    \hspace{0.3in}
    \forall j \neq i, \ell\in\{0\}\cup[N], \nu'_{j,\ell} = \nu_{j,\ell}.
\end{align*}
Now arm $i$ is not the optimal arm under the instance $\nu'$ since it has lower performance than arm $k$. Under the event $\calA$ and let $\alpha \to 0$, we have
\begin{align*}
    \bbE_\nu[N_{i,0}(\tau)] \cdot\frac{(\max_{j\in\feasible\backslash\{\iopt\}}\mu_{j,0}-\mu_{i,0})^2}{2}\ge\log\frac{1}{2.4\delta}.
\end{align*}
Recall the definition of $\phi_i =(\mu_{i,0} - \max_{j\in\feasible\backslash\{\iopt\}}\mu_{j,0})^{-2}$ for arm $\iopt$ in Equation~\eqref{eqn:phii}, we have
\begin{align*}
    \bbE_\nu[N_{i,0}(\tau)] \ge 2\phi_i \log\frac{1}{2.4\delta}.
\end{align*}
Putting the two cases together, we have
\begin{align*}
    \sum_{\ell=0}^N\bbE_{\nu}[N_{\iopt,\ell}(\tau)] \ge 2(\theta_{\iopt} + \phi_{\iopt})\log\frac{1}{2.4\delta}.
\end{align*}
Putting three cases together for arms in $\calI, \calW$ and $\{\iopt\}$, and recall the definition of $\calH$ in Equation~\eqref{eqn: H}, we have
\begin{align*}
    \bbE_{\nu}[\tau] = \sum_{i=1}^K \sum_{\ell=0}^N \bbE_\nu[N_{i,\ell}(\tau)] \ge 2\calH\log\frac{1}{2.4\delta}.
\end{align*}

This completes the proof for when the feasible set is non-empty, i.e. $\iopt\leq K$.
Next, we will consider the case when there is no feasible arm, i.e. $\iopt = K+1$.

\paragraph{When there is no feasible arm, i.e $\iopt = K+1$.} It indicates that all arms are infeasible. This is similar to the analysis in \emph{Case A.1}. For each arm $i \in [K]$, we can construct the alternate instance in which arm $i$ becomes feasible. It is easy to see that for arm $i$, we have
\begin{align*}
    \sum_{\ell=0}^N \bbE_{\nu}[N_{i,\ell}(\tau)]\ge2\theta_i\log\frac{1}{2.4\delta}.
\end{align*}
Thus combining all arms and recall the definition of $\calH$ in Equation~\eqref{eqn: H}, we have
\begin{align*}
    \bbE_\nu[\tau] = \sum_{i=1}^K\sum_{\ell=0}^N\bbE_\nu[N_{i,\ell}(\tau)] \ge 2\calH\log\frac{1}{2.4\delta}.
\end{align*}


\section{Proof of upper Bound (Theorem~\ref{thm: upper bound})}

In this section, we will prove Theorem~\ref{thm: upper bound}. We will split it into two parts: 
\begin{enumerate}
    \item In Section Proof of correctness, we prove that Algorithm~\ref{alg:ours} is $\delta$-correct in Lemma~\ref{appendix: lemma:delta-correct} by constructing a clean event, lower bounding the probability of this clean event by $1-\delta$, and then showing that the algorithm will output $\iopt$ under this clean event.
    \item In Section Upper bound on the sample complexity, we upper bound of the sample complexity of Algorithm~\ref{alg:ours} by decomposing the total number of epochs $\tau$ based on whether the arms pulled are in $\calI$ or $\calW$ and bounding each part separately in Lemma~\ref{lemma: sample complexity}. 
\end{enumerate}
\begin{proof}[Proof of Theorem~\ref{thm: upper bound}]
    Theorem~\ref{thm: upper bound} follows
    by combing Lemma~\ref{appendix: lemma:delta-correct} in Section Proof of correctness  and Lemma~\ref{lemma: sample complexity} in Section Upper bound on the sample complexity.
\end{proof}

\subsection{Proof of correctness} \label{appendix:sec correctness}

In this section, we will prove the following lemma.

\begin{lemma} \label{appendix: lemma:delta-correct}
    Algorithm~\ref{alg:ours} is $\delta$-correct.
    That is, on any bandit instance satisfying the assumptions in Section PROBLEM SETUP: let $\delta >0$ and let $\widehat{a}$ denote the recommendation arm of Algorithm~\ref{alg:ours}, recall that $\iopt$ denote the optimal arm as defined in Equation~\eqref{eqn:ioptdefn}, $\bbP(\widehat{a} = \iopt) \ge 1-\delta$.
\end{lemma}

\paragraph{Proof organization and synopsis:}
This proof is organized as follows. First, in Lemma~\ref{lemma: clean event} we will define a ``clean'' event, and show that this event holds with probability at least $1-\delta$. Using this result, in Lemma~\ref{appendix: lemma:delta-correct}, we will show that the algorithm always identifies the correct $\iopt$ upon termination under this clean event. This technique of proving an algorithm is $\delta$-correct is widely adopted in the bandit literature~\cite{wang2022best,kano2019good}.

\newcommand{\hatmuils}{\widehat{\mu}_{i,\ell,s}}

\begin{lemma} \label{lemma: clean event}
    For any $i \in [K]$ and $\ell \in \{0\}\cup [N]$, let $\hatmuils$ denote the sample mean of arm $i$'s $\ell^{\rm th}$ distribution when the tuple $(i,\ell)$ has been pulled $s$ times. Define the event
    \begin{align*}
        \calE_{i,\ell,s} =  \left\{|\hatmuils - \muil| \le D\left(s, \frac{\delta}{K(N+1)}\right)\right\}
    \end{align*}
    and let
    \begin{align*}
        \calE := \bigcap_{i=1}^K \bigcap_{\ell=0}^N \bigcap_{s=1}^\infty \calE_{i,\ell,s}.
    \end{align*}
    Then it holds that $\bbP(\calE) \ge 1-\delta$.
\end{lemma}
\begin{proof}[Proof of Lemma~\ref{lemma: clean event}]
    By Hoeffding's inequality,
    \begin{align*}
        \bbP(\calE_{i,\ell,s}^\mathsf{c}) \le 2\exp\left(-\frac{s}{2}\cdot \frac{2}{s}\log\frac{4K(N+1)s^4}{\delta}\right) = \frac{\delta}{2K(N+1)s^4}.
    \end{align*}
    By taking a union bound over all $i,\ell,s$, we have,
    \begin{align*}
        \bbP(\calE^c) \le \sum_{i=1}^K \sum_{\ell=0}^N \sum_{s=1}^\infty \frac{\delta}{2K(N+1)s^4} = K\cdot (N+1) \cdot \frac{\delta}{2K(N+1)} \cdot \frac{\pi^2}{6} < \delta.
    \end{align*}
\end{proof}

\newcommand{\ucbmuils}{\overline{\mu}_{i,\ell,s}}
\newcommand{\lcbmuils}{\underline{\mu}_{i,\ell,s}}

\begin{proof}[Proof of Lemma~\ref{appendix: lemma:delta-correct}]
    In this lemma, we show that under the clean event $\calE$, where $\bbP(\calE) \ge 1-\delta$, Algorithm~\ref{alg:ours} will not output a wrong $\iopt$. We consider the cases when $\iopt = K+1$ and $\iopt \le K$ separately. 
    
    First, when $\iopt = K+1$, which indicates that there is no feasible arm. Thus we have that for all arm $i \in [K]$, arm $i$ must contain at least one feasibility constraint $\ell$ such that $\muil > 1/2$. 
    Let $\ucbmuils, \lcbmuils$ denote the Upper (Lower) Confidence bound of arm $i$'s feasibility constraint $\ell$ when the tuple $(i,\ell)$ has been sampled $s$ times, more precisely,
    \begin{align*}
        \ucbmuils &= \hatmuils + D(s,\delta/(K(N+1))),\\
        \lcbmuils &= \hatmuils - D(s, \delta/(K(N+1))).
    \end{align*}
    Note that under the event $\calE$, we have
    \begin{align} \label{Appendix: eq:ucb>1/2}
        \forall s > 0,\; \ucbmuils \ge \muil > 1/2,
    \end{align}
    which means the $\ell^{\rm th}$ feasibility constraint of arm $i$ will never be considered as feasible, see line 27 in Algorithm~\ref{alg:ours}. So arm $i$ will never be considered feasible under the event $\calE$. 
    Thus, this indicates that no arm will be determined as a feasible arm so that no arm in $[K]$ will be outputted as $\iopt$.

    Second, when $\iopt \le K$, which indicates that there exist feasible arms. In this case, we will show that under the clean event $\calE$, arm $\iopt$ will not be eliminated. First, for the arm $\iopt$, it satisfies $\forall \ell \in [N], \mu_{\iopt,\ell} < 1/2$. Under the event $\calE$, we have
    \begin{align*}
        \forall \ell \in [N], s > 0,\; \lcbmuils \le \mu_{\iopt,\ell} < 1/2,
    \end{align*}
    which means all feasibility constraints of arm $\iopt$ will not be considered infeasible, see line 25 in Algorithm~\ref{alg:ours}. This indicates that arm $\iopt$ will not be eliminated due to feasibility.
    Second, we will show that arm $\iopt$ will not be eliminated by performance under the event $\calE$. For an arm $i \in \calI$, as shown in Equation~\eqref{Appendix: eq:ucb>1/2}, arm $i$ will never be considered feasible under the event $\calE$. Since we will only eliminate an arm by performance if it is considered to have worse performance than another arm which is considered feasible, see line 21 in Algorithm~\ref{alg:ours}, we will not eliminate arm $\iopt$ by an arm in $\calI$. On the other hand, for an arm $i \in \calW$, under the event $\calE$, we have
    \begin{align*}
        \forall s_1, s_2 > 0,\quad \overline{\mu}_{\iopt, 0, s_1} \ge \mu_{\iopt, 0} > \mu_{i, 0} \ge \underline{\mu}_{i, 0, s_2}.
    \end{align*}
    Similarly, this indicates arm $\iopt$ will not be eliminated by an arm in $\calW$. Thus putting it together, we show that arm $\iopt$ will not be eliminated by either feasibility or performance under $\calE$.
\end{proof}
\subsection{Upper bound on the sample complexity} \label{appendix:sec sample complexity}

In this section, we will prove the main technical result of this paper, Lemma~\ref{lemma: sample complexity}, which upper bounds the expected number of samples for Algorithm~\ref{alg:ours}.

\begin{lemma} \label{lemma: sample complexity}
The upper bound of the stopping time $\tau$ satisfies Equation~\eqref{eqn: upper bound < K} and ~\eqref{eqn: upper bound > K+1} as stated in Theorem~\ref{thm: upper bound}. 
\end{lemma}


\paragraph{Proof organization:}
This proof is organized as follows. First, in Lemma~\ref{lemma:tau_i}, we upper bound the number of samples required by the algorithm to determine if an infeasible arm is infeasible,
and in Lemma~\ref{lemma:tau_i feasible}, we upper bound the number of samples to determine if a feasible arm is feasible. Second, we decompose the stopping time of Algorithm~\ref{alg:ours} into two parts based on whether the sampled arms at each epoch belong to $\calI$ or not. We then upper bound the two parts separately in Lemma~\ref{lemma:A} and Lemma~\ref{lemma:B}, where we will use the results from Lemma~\ref{lemma:tau_i} and Lemma \ref{lemma:tau_i feasible}.

The main challenge in obtaining an upper bound that matches the lower bound requires a careful decomposition of the \emph{stopping time} $\tau$, i.e. of the number of epochs of Algorithm~\ref{alg:ours} upon termination. 
We decompose $\tau$ into two parts based on whether the sampled arms $a_t, b_t$ belong to $\calI$ or not. Furthermore, in each part, if an arm $i$ is chosen as in line 10 in Algorithm~\ref{alg:ours}, we will treat $a_t = i$ and $b_t$ as undefined.
This decomposition helps us to consider the two different types of elimination (feasibility and worse performance) separately. 
To eliminate an arm by performance, we leverage the insight that if two arms are sampled at an epoch, the confidence intervals of the two arms must overlap. We then upper bound the time for an arm to be eliminated via performance by the sample complexity to differentiate this arm's performance from other arms. 
To eliminate an arm by feasibility, We adapt the analysis from~\citet{kano2019good} to upper bound the expected sample complexity of determining an arm's feasibility. The insight here is that if a constraint of an arm is not determined at an epoch, this constraint's confidence interval must contain the threshold 1/2. 


\begin{proof}[Proof of Lemma~\ref{lemma: sample complexity}]
First, for any $i\in[K]$, we will let $\tau_{F,i}$ denote the number of epochs an arm must have been chosen as either $a_t$ or $b_t$ before it is included in either the set $I$ (line 26) or $F$ (line 30).
As we sample at most one feasibility distribution of the arm on each epoch $t$, and do not sample feasibility distributions after an arm is added to either $F$ or $I$,
this can be viewed as
the number of feasibility samples needed to determine if $i$ is feasible ($i\in\calF$) or not ($i\notin\calF$).
Below, we state two lemmas to bound $\bbE[\tau_{F,i}]$. The proofs of the lemmas can be found in Section~\ref{subsec:feasibility}.
The following lemma bounds this quantity for an infeasible arm. 
\begin{lemma} \label{lemma:tau_i}
Let $i \notin \feasible$.
Let $\tau_{F,i}$ be the number of epochs before $i$ is added to either $F$ or $I$ in Algorithm~\ref{alg:ours}.
We then have,
\begin{align*}
\bbE[\tau_{F,i}] \le &n_{i,1} + \frac{8}{\gamma_i(1,1/2)^2} + \sum_{\ell =1}^N \frac{8N(K(N+1))^{-\frac{1}{8}}}{\gamma_i(\ell,1/2)^2} + \frac{32N - 22+ 16\log\frac{32}{\gamma_i(1,2)^2}}{\gamma_i(1,2)^2} \\
    +&\sum_{\ell \neq 1} \left( 
    \frac{8\log T_i}{\gamma_i(1,\ell)^2}  + \frac{\delta}{K(N+1)} \left(n_{i,\ell} + \frac{8}{\gamma_i(\ell,1/2)^2} \right)
    \right)  
\end{align*}
    where
    \begin{align*}
        n_{i,\ell} = \frac{8}{(\gamma_i(\ell,1/2)/2)^2} \log \frac{8^{\frac{3}{2}} (K(N+1)/\delta)^{\frac{1}{4}} }{(\gamma_i(\ell,1/2)/2)^2} \log \frac{(8^{\frac{3}{2}}+1)(K(N+1)/\delta)^{\frac{1}{4}} }{(\gamma_i(\ell,1/2)/2)^2}.
    \end{align*}
\end{lemma}
The following lemma bounds $\tau_{F,i}$ for  a feasible arm.
\begin{lemma} \label{lemma:tau_i feasible} 
Let $i \in \feasible$, and let $\tau_{F,i}$ be as defined in~\ref{lemma:tau_i}. We then have, 
\begin{align*}
    \bbE[\tau_{F,i}] \le \sum_{\ell\in [N]} \left(n_{i,\ell} + \frac{8}{\gamma_i(\ell,1/2)^2}\right).
\end{align*}
\end{lemma}
Next, we will decompose the total number of epochs
as follows.
Recall that $[K] = \calI \cap \{\iopt\} \cap \calW$.
We decompose the stopping time $\tau$ based on whether either $a_t$ or $b_t$ are in $\calI$
(i.e $a_t\in\calI$ or $b_t\in\calI$) 
or if both of them are not in $\calI$
(i.e $a_t\notin\calI$ and $b_t\notin\calI$).
The reason for this decomposition lies in the difference of elimination for arms in $\calI$ and not, i.e. arms in $\calI$ should be eliminated by infeasibility while arms in $\calW$ should be eliminated by performance. 
\begin{align*}
    \tau &= \sum_{t=1}^\infty \msi[t \le \tau] \\
    &= \underbrace{\sum_{t=1}^\infty \msi[t \le \tau
        \cap \left(a_t \in \calI \cup b_t \in \calI\right)]}_{A} 
        + \underbrace{\sum_{t=1}^\infty \msi[t \le \tau \cup a_t \notin \calI \cup b_t \notin \calI]}_{B}.
\end{align*}
The following lemma bounds part $A$. The proof can be found in Section~\ref{subsec: A}
\begin{lemma} \label{lemma:A}
\begin{align*}
    &\bbE[A] \le \sum_{i\in\calI}\Bigg(\bbE[\tau_{F,i}] \\
    &+ 2\delta\sum_{j\neq i}2\left(\frac{8}{(\Delta_{ij}/4)^2} \log\frac{8^{\frac{3}{2}}(K(N+1)/\delta)^{\frac{1}{4}}}{(\Delta_{ij}/4)^2}\log\frac{(8^{\frac{3}{2}}+1) (K(N+1)/\delta)^{\frac{1}{4}}}{(\Delta_{ij}/4)^2} + \frac{\delta}{K(N+1)}\right)\Bigg)
\end{align*}
\end{lemma}
The following lemma bounds part $B$. The proof can be found in Section~\ref{subsec:B}.
\begin{lemma} \label{lemma:B}
\begin{align*}
    &\bbE[B] \\
    \le &\delta\sum_{i \in \calW} \bbE[\tau_{F,i}] + \bbE[\tau_{F,\iopt}] +292\sum_{i\notin\calI}\phi_i\log\frac{\sum_{i\notin\calI \phi_i}}{\delta}  + 16\\
    + &\delta \sum_{i, j \notin \calI} 2\left(\frac{8}{(\Delta_{ij}/4)^2} \log\frac{8^{\frac{3}{2}}(K(N+1)/\delta)^{\frac{1}{4}}}{(\Delta_{ij}/4)^2}\log\frac{(8^{\frac{3}{2}}+1) (K(N+1)/\delta)^{\frac{1}{4}}}{(\Delta_{ij}/4)^2} + \frac{\delta}{K(N+1)}\right).
\end{align*}
\end{lemma}

We are now ready to bound the stopping time.
We will first consider the case when the feasible set is non-empty, i.e $\iopt \le K$.
By Lemma~\ref{lemma:A} and Lemma~\ref{lemma:B}, combining part $A$ and part $B$, we have   
\begin{align*}
    &\bbE[\tau]\\
    \le &\sum_{i\in\calI}\Bigg(\bbE[\tau_{F,i}] \\
    &+ 2\delta\sum_{j\neq i}2\left(\frac{8}{(\Delta_{ij}/4)^2} \log\frac{8^{\frac{3}{2}}(K(N+1)/\delta)^{\frac{1}{4}}}{(\Delta_{ij}/4)^2}\log\frac{(8^{\frac{3}{2}}+1) (K(N+1)/\delta)^{\frac{1}{4}}}{(\Delta_{ij}/4)^2} + \frac{\delta}{K(N+1)}\right)\Bigg)\\
    +&\delta\sum_{i \in \calW} \bbE[\tau_{F,i}] + \bbE[\tau_{F,\iopt}] +292\sum_{i\notin\calI}\phi_i\log\frac{\sum_{i\notin\calI} \phi_i}{\delta}  + 16\\
    \numberthis + &\delta \sum_{i, j \notin \calI} 2\left(\frac{8}{(\Delta_{ij}/4)^2} \log\frac{8^{\frac{3}{2}}(K(N+1)/\delta)^{\frac{1}{4}}}{(\Delta_{ij}/4)^2}\log\frac{(8^{\frac{3}{2}}+1) (K(N+1)/\delta)^{\frac{1}{4}}}{(\Delta_{ij}/4)^2} + \frac{\delta}{K(N+1)}\right). \label{eqn:bbEtau,1}
\end{align*}
To bound $\bbE[\tau_{F,i}]$ for arms in $\calI$, we apply Lemma~\ref{lemma:tau_i}. To bound $\bbE[\tau_{F,i}]$ for arms in $\calW$ and $\bbE[\tau_{F, \iopt}]$, we apply Lemma~\ref{lemma:tau_i feasible}.
Thus, when $\iopt \le K$, by putting Lemma~\ref{lemma:tau_i} and Lemma~\ref{lemma:tau_i feasible} into Equation~\eqref{eqn:bbEtau,1},
we have
\begin{align*}
    &\bbE[\tau] \le \sum_{i\in\calI} f\left(\frac{\gamma_i(1,1/2)}{2}\right)  + 292\sum_{i\notin\calI}\phi_i\log\frac{\sum_{i\notin\calI} \phi_i}{\delta} + \sum_{\ell =1}^N f\left(\frac{\gamma_{\iopt}(\ell,1/2)}{2}\right)  \\
    & + \sum_{\ell=1}^N \frac{8}{\gamma_{\iopt}(\ell, 1/2)^2} + 16 + \sum_{i\in\calI} g_i\\
    &+ \delta \Bigg( 4\sum_{i\in\calI}\sum_{j\neq i} \left(f\left(\frac{\Delta_{ij}}{4}\right) + \frac{\delta}{K(N+1)}\right) + \sum_{i\in\calW}\sum_{\ell=1}^N \left(n_{i,\ell} + \frac{8}{\gamma_i(\ell,1/2)^2}\right) \\
    &+ 2\sum_{i,j\notin \calI}\left(f\left(\frac{\Delta_{ij}}{4}\right) + \frac{\delta}{K(N+1)} \right) \Bigg)
\end{align*}
where 
\begin{align*}
    f(x) &= \frac{8}{x^2} \log \frac{8^{\frac{3}{2}}(K(N+1)/\delta)^{\frac{1}{4}}}{x^2}\log\frac{(8^{\frac{3}{2}}+1)(K(N+1)/\delta)^{\frac{1}{4}}}{x^2}, \\
    g_i &= \frac{8}{\gamma_i(1,1/2)^2} + \sum_{\ell =1}^N \frac{8N(K(N+1))^{-\frac{1}{8}}}{\gamma_i(\ell,1/2)^2} + \frac{32N - 22+ 16\log\frac{32}{\gamma_i(1,2)^2}}{\gamma_i(1,2)^2} \\
    +&\sum_{\ell \neq 1} \left( 
    \frac{8\log T_i}{\gamma_i(1,\ell)^2}  + \frac{\delta}{K(N+1)} \left(n_{i,\ell} + \frac{8}{\gamma_i(\ell,1/2)^2} \right)
    \right),
\end{align*}
and see definitions of $\phi_i, \Delta_{ij}$ in Table 1 and $n_{i,\ell}, \gamma_{i}(\ell,m), \gamma_{i}(\ell,1/2)$ in Table 2.
Then by rearranging the terms, we can write $\bbE[\tau]$ in the form presented in Theorem~\ref{thm: upper bound}. 

When $\iopt = K+1$, this means that $\calI = [K]$. Hence, the  upper bound of $\bbE[\tau] = \bbE[A]$, which is proved in Lemma~\ref{lemma:A}. 
\end{proof}

\subsubsection{Proof of Lemma~\ref{lemma:A} \emph{(Bound on $\bbE[A]$)} }  \label{subsec: A}

\begin{proof}[Proof of Lemma~\ref{lemma:A}]
For part $A$, by considering each arm $i \in \calI$ separately, we can further decompose it into the cases that whether arm $i$'e feasibility has been determined or not. Here we will use $F(t)$ and $I(t)$ to denote the set of arms considered feasible (infeasible) at epoch $t$. When arm $i$'s feasibility has not been determined at epoch $t$, it means $i \notin F(t)\cup I(t)$. When arm $i$'s feasibility has been determined at epoch $t$, note that since it is still be chosen as $a_t$ or $b_t$, this means arm $i$ can only be considered feasible otherwise it will be eliminated already, thus $i \in F(t)$. In other words, we have
\begin{align*}
    A &= \sum_{t=1}^\infty \msi[t\le\tau, a_t \mbox{ or } b_t \in \calI] \\
    &= \sum_{i \in \calI} \sum_{t=1}^\infty \msi [t\le\tau,a_t \mbox{ or } b_t = i] \\
    &= \sum_{i \in \calI} \left(\underbrace{\sum_{t=1}^\infty  \msi[t \le \tau,  a_t \mbox{ or } b_t = i, i \notin F(t) \cup I(t)]}_{A_1} + \underbrace{\sum_{t=1}^\infty \msi[t\le\tau, a_t \mbox{ or } b_t = i, i \in F(t)]}_{A_2}\right)
\end{align*}
For each arm $i$ in $\calI$, $A_1$ denote the number of epochs arm $i$ chosen to be sampled when its feasibility has not been determined; $A_2$ denote the number of epochs arm $i$ is chosen to be sampled when its feasibility is determined.

We will bound the two parts separately. We prove the upper bound of $A_1$ in Lemma~\ref{lemma: A_1}, where we upper bound it by the number of samples needed to determine arm $i$'s feasibility. 
\begin{lemma}\emph{(Bound on $\bbE[A_1]$)}
\label{lemma: A_1}
\begin{align*}
    \bbE[A_1] \le \bbE[\tau_{F,i}].
\end{align*}
\end{lemma}
We prove the upper bound of $A_2$ in Lemma~\ref{lemma: A_2}, where we upper bound it by splitting
$\{a_t \mbox{ or } b_t = i, i \in F(t)\}$ into two events $\{a_t \mbox{ or } b_t = i\}$, $\{i\in F(t)\}$ and bound the two events separately.
\begin{lemma}\emph{(Bound on $\bbE[A_2]$)}
\label{lemma: A_2}
\begin{align*}
    &\bbE[A_2] \\
    &\le 2\delta \sum_{j \neq i} 2\left(\frac{8}{(\Delta_{ij}/4)^2} \log\frac{8^{\frac{3}{2}}(K(N+1)/\delta)^{\frac{1}{4}}}{(\Delta_{ij}/4)^2}\log\frac{(8^{\frac{3}{2}}+1) (K(N+1)/\delta)^{\frac{1}{4}}}{(\Delta_{ij}/4)^2} + \frac{\delta}{K(N+1)}\right)
\end{align*}
\end{lemma}
The proofs of the lemmas can be found in Section~\ref{subsec: A}. Combining Lemma~\ref{lemma: A_1} and Lemma~\ref{lemma: A_2}, we finish the proof to bound $\bbE[A]$.
\end{proof}

\begin{proof}[Proof of Lemma~\ref{lemma: A_1}]
For part $A_1$, 
since arm $i$'s feasibility has not been determined yet at epoch $t$, the number of times it has been sampled for performance and the total number of times it has been sampled for feasibility are the same, i.e. $N_{i, 0}(t) = M_i(t)$. We have
\begin{align*}
    A_1 &= \sum_{t=1}^\infty  \msi[t \le \tau,  a_t \mbox{ or } b_t = i, i \notin F(t) \cup I(t)] \\
    &\le \sum_{t=1}^\infty \sum_{s=1}^\infty \msi[a_t \mbox{ or }b_t = i, i \notin F(t) \cup I(t), M_i(t) = s]
\end{align*}
Since the event $\{a_t \mbox{ or }b_t = i, i \notin F(t) \cup I(t), M_i(t) = s\}$ occurs for at most one $t \in \bbN$, we have
\begin{align*}
    &\sum_{t=1}^\infty \sum_{s=1}^\infty \msi[a_t \mbox{ or }b_t = i, i \notin F(t) \cup I(t), M_i(t) = s]\\
    \le &\sum_{s=1}^\infty \msi\left[\bigcup_{t=1}^\infty \{a_t \mbox{ or }b_t = i, i \notin F(t)\cup I(t), M_i(t) = s\}\right] \\
    \le &\sum_{s=1}^\infty \msi[i\notin( F \cup I)_s ],
\end{align*}
where $i \notin (F \cup I)_s$ denotes that arm $i$'s feasibility has not been determined when its feasibility constraints have been sampled $s$ times. Let $\tau_{F,i}$ denote the number of samples to feasibility constraints to determine arm $i$'s feasibility, i.e. $\tau_{F,i} = \sum_{s=1}^\infty \msi[i\notin (F\cup I)_s]$, we have that for each arm $i \in \calI$,
\begin{align*}
    \bbE[A_1] \le \bbE[\tau_{F,i}].
\end{align*}
\end{proof}

\begin{proof}[Proof of Lemma~\ref{lemma: A_2}]
For part $A_2$, for each arm $i \in \calI$, we first split the event $\{a_t \mbox{ or }b_t = i\}$ to $\{a_t=i\}$ and $\{b_t = i\}$,
\begin{align*}
    A_2 &= \sum_{t=1}^\infty \msi[t\le\tau, a_t \mbox{ or } b_t = i, i \in F(t)] \\
    &\le \underbrace{\sum_{t=1}^\infty \msi[a_t = i, i \in F(t)]}_{A_3} + \underbrace{\sum_{t=1}^\infty \msi[b_t = i, i \in F(t)]}_{A_4}
\end{align*}
We will bound both $A_3$ and $A_4$ using similar techniques, so we will only show this for part $A_3$. 
Note that in Algorithm~\ref{alg:ours}, there are two situations to sample arms: in line 10 we will only sample one arm $i \in P, i \notin F$ and consider $a_t = i$, $b_t$ undefined; in line 13 we will sample two arms $a_t,b_t$. However, since here $i \in F(t)$, it indicates we cannot be in line 10. So following line 13, we will also select another arm $b_t \in [K] \backslash\{a_t\}$. We have:
\begin{align*}
    A_3 = \sum_{t=1}^\infty \msi[a_t = i, i \in F(t)] = \sum_{t=1}^\infty \sum_{j\neq i} \msi[a_t = i, b_t = j, i \in F(t)]
\end{align*}
Next, we split the event $F(t)$ out because later on we want to bound its probability separately, by union bound we have
\begin{align*}
    \sum_{t=1}^\infty \sum_{j\neq i} \msi[a_t = i, b_t = j, i \in F(t)]
    &\le \sum_{t=1}^\infty \sum_{j\neq i} \msi[a_t = i, b_t = j]\cdot \msi[i \in F(t)] \\
    &\le \underbrace{\left(\sum_{t=1}^\infty \sum_{j \neq i} \msi[a_t = i, b_t = j]\right)}_{A_5} \cdot \msi\left[\bigcup_{t'=1}^\infty \{i \in F(t')\}\right] 
\end{align*}
Now we bound part $A_5$, let $N_{ij}^{ab}(t) = \sum_{s=1}^t\msi[a_s=i, b_s = j]$ denote the number of times arms $i$ and $j$ have respectively been selected as $a_s$ and $b_s$ up to epoch $t$, we have
\begin{align*}
    A_5 = \sum_{t=1}^\infty \sum_{j \neq i} \msi[a_t = i, b_t = j] = \sum_{t=1}^\infty \sum_{j \neq i} \sum_{s=1}^\infty \msi[a_t = i, b_t = j, N_{ij}^{ab}(t) = s]
\end{align*}
Similar to the strategy we use to bound $A_1$, since the event $\{a_t =i, b_t = j, N_{ij}^{ab}(t) = s\}$ occurs for at most one $t\in \bbN$, we have
\begin{align*}
    \sum_{t=1}^\infty \sum_{j \neq i} \sum_{s=1}^\infty \msi[a_t = i, b_t = j, N_{ij}^{ab}(t) = s] 
    \le \sum_{j \neq i} \sum_{s=1}^\infty \msi\left[\bigcup_{t=1}^\infty \{a_t = i, b_t = j, N_{ij}^{ab}(t) = s\}\right]
\end{align*}
Let $L_{is}$ and $L_{js}$ denote the number of times arm $i$ and $j$ respectively has been sampled on performance when $N_{ij}^{ab}(t) = s$. Note that $L_{is}$ and $L_{js}$ are random variables, and we will apply a union bound later to bound these two terms. 
The event $\{a_t = i, b_t = j\}$ indicates that we do not know which arm is better between arm $i$ and $j$ at epoch $t$, in other words, their confidence intervals on the performance overlap. By separating into the cases when $i < j$ and $i > j$,  we have
\begin{align*}
    &\sum_{j \neq i} \sum_{s=1}^\infty \msi\left[\bigcup_{t=1}^\infty \{a_t = i, b_t = j, N_{ij}^{ab}(t) = s\}\right]\\
    \le &\sum_{j > i} \sum_{s=1}^\infty \msi[\underline{\mu}_{i, 0, L_{is}} < \overline{\mu}_{j, 0, L_{js}}] + \sum_{j < i} \sum_{s=1}^\infty \msi[\underline{\mu}_{j, 0,L_{js}} < \overline{\mu}_{i, 0, L_{is}}] 
\end{align*}
Note that the cases when $j<i$ and $j>i$ are indeed the same, so we will only prove for the case when $j<i$.  
Now for all $j > i$, which indicates $\mu_{i,0}>\mu_{j,0}$. Let $\Delta_{ij} = |\mu_{i,0}-\mu_{j,0}|$, we can upper bound the inner summation by
\begin{align}
    \nonumber \sum_{s=1}^\infty \msi[\underline{\mu}_{i, 0, L_{is}} < \overline{\mu}_{j, 0, L_{js}}] 
    &\le \sum_{s=1}^\infty \msi[\underline{\mu}_{i, 0, L_{is}} < \mu_{i,0} - \Delta_{ij}/2 \mbox{ or }\overline{\mu}_{j, 0, L_{js}} > \mu_j + \Delta_{ij}/2] \\
    \nonumber &\le \sum_{s=1}^\infty \msi[\underline{\mu}_{i, 0, L_{is}} < \mu_{i,0} - \Delta_{ij}/2] + \sum_{s=1}^\infty \msi[\overline{\mu}_{j, 0, L_{js}} > \mu_j + \Delta_{ij}/2] \\
    &= 2\sum_{s=1}^\infty \msi[\underline{\mu}_{i, 0, L_{is}} < \mu_{i,0} - \Delta_{ij}/2], \label{eqn:j>i}
\end{align}
since the two terms related to $i$ and $j$ are exactly the same.
Now, we take an expectation and choose a $u > 0$ large enough so that $\Delta_{ij} > 4\sqrt{\frac{2\log(4K(N+1)u^4/\delta)}{u}}$, by Lemma~\ref{lemma: n ge n_i}, it implies
\begin{align*}
    \numberthis u = \frac{8}{(\Delta_{ij}/4)^2} \log\frac{8^{\frac{3}{2}}(K(N+1)/\delta)^{\frac{1}{4}}}{(\Delta_{ij}/4)^2}\log\frac{(8^{\frac{3}{2}}+1) (K(N+1)/\delta)^{\frac{1}{4}}}{(\Delta_{ij}/4)^2} \label{eqn: u}
\end{align*}
we have
\begin{align}
    \nonumber&\bbE\left[\sum_{s=1}^\infty \msi[\underline{\mu}_{i, 0, L_{is}} <\mu_{i,0} - \Delta_{ij}/2]\right] \\
    \nonumber \le &u + \sum_{s=u+1}^\infty \bbP\left(\underline{\mu}_{i, 0, L_{is}}< \mu_{i,0} - \Delta_{ij}/2\right) \\
    \nonumber\le &u + \sum_{s=u+1}^\infty \sum_{v=s}^\infty \bbP \left(\underline{\mu}_{i,0,v} < \mu_{i,0} - \Delta_{ij}/2 \right) \hspace{1in} (\mbox{by Union Bound and }L_{is} \ge s) \\
    = &u + \sum_{s=u+1}^\infty \sum_{v=s}^\infty \bbP \left(\widehat{\mu}_{i,0,v} -\sqrt{\frac{2\log(4K(N+1)v^4/\delta)}{v}} < \mu_{i,0}  - \Delta_{ij}/2 \right) \label{eqn:A_5}
\end{align}
Since $v > u$, we have $\Delta_{ij} > 4\sqrt{\frac{2\log(4K(N+1)v^4/\delta)}{v}}$, so that
\begin{align*}
    &\sum_{s=u+1}^\infty \sum_{v=s}^\infty \bbP \left(\widehat{\mu}_{i,0,v} -\sqrt{\frac{2\log(4K(N+1)v^4/\delta)}{v}} < \mu_{i,0}  - \Delta_{ij}/2 \right) \\
    \le &\sum_{s=u+1}^\infty \sum_{v=s}^\infty \bbP \left(\widehat{\mu}_{i,0,v}  < \mu_{i,0} - \sqrt{\frac{2\log(4K(N+1)v^4/\delta)}{v}} \right) \\
    \le &\sum_{s=u+1}^\infty \sum_{v=s}^\infty \exp\left(-\frac{v}{2}\cdot \frac{2\log(4K(N+1)v^4/\delta)}{v}\right) \hspace{0.3in} (\mbox{by Hoeffding's inequality}) \\
    = &\sum_{s=u+1}^\infty \sum_{v=s}^\infty \frac{\delta}{4K(N+1)v^4}.
\end{align*}
Note that 
\begin{align*}
    \sum_{s=1}^\infty \sum_{v=s}^\infty \frac{1}{v^4} = \sum_{s=1}^\infty \left(\frac{1}{s^4} + \frac{1}{(s+1)^4} + \cdots \right) = \sum_{s=1}^\infty \frac{1}{s^3} \le \frac{\pi^2}{6}.
\end{align*}
Putting it back to Equation~\eqref{eqn:A_5} and combining Equation~\eqref{eqn: u}, we have
\begin{align*}
    &\bbE\left[\sum_{s=1}^\infty \msi[\underline{\mu}_{i,0L_{is}} < \mu_{i,0} - \Delta_{ij}/2]\right] \\
    \le &\frac{8}{(\Delta_{ij}/4)^2} \log\frac{8^{\frac{3}{2}}(K(N+1)/\delta)^{\frac{1}{4}}}{(\Delta_{ij}/4)^2}\log\frac{(8^{\frac{3}{2}}+1) (K(N+1)/\delta)^{\frac{1}{4}}}{(\Delta_{ij}/4)^2}+ \frac{\delta}{K(N+1)}.
\end{align*}
Putting it back to Equation~\eqref{eqn:j>i}, we have
\begin{align*}
    &\bbE\left[\sum_{j>i}\sum_{s=1}^\infty \msi[\underline{\mu}_{i,0,L_{is}}<\overline{\mu}_{j,0,L_{js}}]\right]\\
    \le &\sum_{j > i} 2\left(\frac{8}{(\Delta_{ij}/4)^2} \log\frac{8^{\frac{3}{2}}(K(N+1)/\delta)^{\frac{1}{4}}}{(\Delta_{ij}/4)^2}\log\frac{(8^{\frac{3}{2}}+1) (K(N+1)/\delta)^{\frac{1}{4}}}{(\Delta_{ij}/4)^2} + \frac{\delta}{K(N+1)}\right)
\end{align*}
Similarly for $j < i$, which indicates $\mu_{i,0} < \mu_{j,0}$. For all $j < i$, we have
\begin{align*}
    \sum_{s=1}^\infty \msi[\underline{\mu}_{j, 0,L_{js}} < \overline{\mu}_{i, 0, L_{is}}] < \sum_{s=1}^\infty \msi[\underline{\mu}_{j,0,L_{js}} < \mu_{j,0} - \Delta_{ij}/2] + \sum_{s=1}^\infty \msi[\overline{\mu}_{i,0,L_{is}} >\mu_{i,0} + \Delta_{ij}/2]
\end{align*}
Note that this is similar to Equation~\eqref{eqn:j>i}. Thus using the same strategy, we have
\begin{align*}
    &\bbE\left[\sum_{j < i} \sum_{s=1}^\infty \msi[\underline{\mu}_{j,0,L_{js}} < \overline{\mu}_{i,0,L_{is}}]\right]\\
    \le &\sum_{j < i} 2\left(\frac{8}{(\Delta_{ij}/4)^2} \log\frac{8^{\frac{3}{2}}(K(N+1)/\delta)^{\frac{1}{4}}}{(\Delta_{ij}/4)^2}\log\frac{(8^{\frac{3}{2}}+1) (K(N+1)/\delta)^{\frac{1}{4}}}{(\Delta_{ij}/4)^2} + \frac{\delta}{K(N+1)}\right)
\end{align*}
Thus we have
\begin{align*}
    \bbE[A_5] \le \sum_{j \neq i} 2\left(\frac{8}{(\Delta_{ij}/4)^2} \log\frac{8^{\frac{3}{2}}(K(N+1)/\delta)^{\frac{1}{4}}}{(\Delta_{ij}/4)^2}\log\frac{(8^{\frac{3}{2}}+1) (K(N+1)/\delta)^{\frac{1}{4}}}{(\Delta_{ij}/4)^2} + \frac{\delta}{K(N+1)}\right)
\end{align*}
Note that $\bbP(\bigcup_{t'=1}^\infty \{i\in F(t')\}) < \delta$. Thus we have
\begin{align*}
    &\bbE[A_3] \\
    \le &\delta \sum_{j \neq i} 2\left(\frac{8}{(\Delta_{ij}/4)^2} \log\frac{8^{\frac{3}{2}}(K(N+1)/\delta)^{\frac{1}{4}}}{(\Delta_{ij}/4)^2}\log\frac{(8^{\frac{3}{2}}+1) (K(N+1)/\delta)^{\frac{1}{4}}}{(\Delta_{ij}/4)^2}+ \frac{\delta}{K(N+1)}\right)
\end{align*}
Note that $A_4$ is indeed identical to $A_3$, thus we have
\begin{align*}
    &\bbE[A_2] \\
    \le &\bbE[A_3] + \bbE[A_4] \\
    \le &2\delta \sum_{j \neq i} 2\left(\frac{8}{(\Delta_{ij}/4)^2} \log\frac{8^{\frac{3}{2}}(K(N+1)/\delta)^{\frac{1}{4}}}{(\Delta_{ij}/4)^2}\log\frac{(8^{\frac{3}{2}}+1) (K(N+1)/\delta)^{\frac{1}{4}}}{(\Delta_{ij}/4)^2} + \frac{\delta}{K(N+1)}\right).
\end{align*}
\end{proof}

\subsubsection{Proof of Lemma~\ref{lemma:B} \emph{(Bound on $\bbE[B]$)}} 
\label{subsec:B}


\begin{proof}[Proof of Lemma~\ref{lemma:B}]
Recall that in Algorithm~\ref{alg:ours}, there are two situations to sample arms: in line 10 we will only sample one arm $i$ when $i \in P, i \notin F$ where we consider $a_t = i$, $b_t$ undefined; in line 13 we will sample two arms $a_t,b_t$.
So here
we first split the event that $\{a_t \mbox{ and }b_t \notin \calI\}$ into the cases that $b_t$ is defined or not, which corresponds to line 10 and line 13 in Algorithm~\ref{alg:ours}. We have
\begin{align*}
    B &= \sum_{t=1}^\infty \msi[t\le \tau, a_t \mbox{ and }b_t \notin \calI] \\
    &= \underbrace{\sum_{t=1}^\infty \msi[t\le \tau, a_t \notin \calI, b_t \mbox{ undefined}]}_{B_1} + \underbrace{\sum_{s=1}^\infty \msi[t\le\tau, a_t \notin \calI, b_t \notin R ]}_{B_2}
\end{align*}
We bound the two parts $B_1, B_2$ separately. In Lemma~\ref{lemma: B_1}, we upper bound $\bbE[B_1]$ by the sample complexity to determine the feasibility of each arm not in $\calI$.
\begin{lemma} (\emph{Bound on $\bbE[B_1]$})
\label{lemma: B_1}
\begin{align*}
    \bbE[B_1] \le \delta\sum_{i \in \calW} \bbE[\tau_{F,i}] + \bbE[\tau_{F,\iopt}].
\end{align*}
\end{lemma}
In Lemma~\ref{lemma: B_2}, we upper bound $\bbE[B_2]$ by the sample complexity to differentiate the performance of arm $a_t$ and $b_t$.
\begin{lemma} (\emph{Bound on $\bbE[B_2]$})
\label{lemma: B_2}
\begin{align*}
    &\bbE[B_2] \\
    \le &\delta \sum_{i, j \notin \calI} 2\left(\frac{8}{(\Delta_{ij}/4)^2} \log\frac{8^{\frac{3}{2}}(K(N+1)/\delta)^{\frac{1}{4}}}{(\Delta_{ij}/4)^2}\log\frac{(8^{\frac{3}{2}}+1) (K(N+1)/\delta)^{\frac{1}{4}}}{(\Delta_{ij}/4)^2} + \frac{\delta}{K(N+1)}\right)\\
    + &292\sum_{i\notin\calI}\phi_i\log\frac{\sum_{i\notin\calI \phi_i}}{\delta}  + 16.
\end{align*}
\end{lemma}
The proofs of the two lemmas can be found in Section~\ref{subsec:B}. Combining Lemma~\ref{lemma: B_1} and Lemma~\ref{lemma: B_2}, we finish the proof of upper bounding $\bbE[B]$.
\end{proof}

\begin{proof}[Proof of Lemma~\ref{lemma: B_1}]
Note that $B_1$ contains the event $a_t \notin \calI$ and $\calI^\mathsf{c} = \{\iopt\}\cup\calW$. So we split $B_1$ into two parts $B_3$ and $B_4$ which consider different arms that are not in $\calI$, i.e. $B_3$ considers the case when $\{a_t = \iopt\}$ and $B_4$ considers the case when $\{a_t \in \calW\}$. We will bound parts $B_3$ and $B_4$ separately by the insight that $b_t$ is undefined implies Algorithm~\ref{alg:ours} enters line 10 rather than line 13 to sample arms, and since line 10 requires the arm is not in $F$ then we can bound via the sample complexity to determine an arm's feasibility.
\begin{align*}
    B_1 &= \sum_{t=1}^\infty \msi[t\le\tau,a_t\notin \calI, b_t \mbox{ undefined}] \\
    &=\underbrace{\sum_{i \in \calW} \sum_{t=1}^\infty \msi[t\le\tau,a_t=i, b_t \mbox{ undefined}]}_{B_3} + \underbrace{\sum_{t=1}^\infty \msi[t\le\tau, a_t = \iopt, b_t \mbox{ undefined}]}_{B_4}
\end{align*}

For part $B_3$, the event $\{a_t = i, b_t \mbox{ undefined}\}$ and the fact that $a_t \in \calW$ implies that the focus set at epoch $t$, $P(t) = \{a_t\}$ and $a_t \notin F(t) \cup I(t)$, see line 10 in Algorithm~\ref{alg:ours}. Note that $P(t) = \{a_t
\}$ implies that at some epoch $t'<t$, arm $a_t$ is considered better than arm $\iopt$ on performance or arm $\iopt$ has been eliminated as an infeasible arm. This will only happen with probability less than $\delta$ since we have proved that Algorithm~\ref{alg:ours} is a $\delta$-correct algorithm. Thus we have
\begin{align*}
    B_3 &= \sum_{i \in \calW} \sum_{t=1}^\infty \msi[t\le\tau,a_t=i, b_t \mbox{ undefined}] \\
    &\le \sum_{i \in \calW} \sum_{t=1}^\infty \msi[a_t = i, i \notin F(t)\cup I(t), P(t) = \{i\}] \\
    &\le \sum_{i\in \calW} \underbrace{\left(\sum_{t=1}^\infty \msi[a_t = i,i \notin F(t)\cup I(t)]\right)}_{B_5} \cdot \msi\left[\bigcup_{t'=1}^\infty \{P(t') = i\}\right]
\end{align*}
Now we bound the expectation of $B_5$, $a_t = i$ can first be upper bouneded by $a_t \mbox{ or }b_t = i$. Then similar to the process of bounding $A_1$, see Lemma~\ref{lemma: A_1}, $B_5$ can be upper bounded by the sample complexity required to determine arm $i$'s feasibility, we have
\begin{align*}
    \bbE[B_5] &= \bbE\left[\sum_{t=1}^\infty \msi[a_t = i,i \notin F(t)\cup I(t)]\right] \\
    &\le \bbE\left[\sum_{t=1}^\infty \msi[a_t \mbox{ or }b_t = i, i \notin F(t)\cup I(t)]\right] \\
    &\le \bbE[\tau_{F,i}] \hspace{2in} (\mbox{see Lemma~\ref{lemma: A_1}})
\end{align*}
It is easy to see that $\bbP\left(\bigcup_{t'=1}^\infty \{P(t') = i\}\right) < \delta$, thus the expectation of $B_3$ is bounded by
\begin{align*}
    \bbE[B_3] \le \delta\sum_{i \in \calW}\bbE[\tau_{F,i}]. 
\end{align*}
Next, we bound part $B_4$, similar to $B_3$, $\{a_t = \iopt, b_t \mbox{ undefined}\}$ implies that $a_t \notin F(t)\cup I(t)$, thus we have
\begin{align*}
    B_4 &= \sum_{t=1}^\infty \msi[t\le\tau, a_t = \iopt, b_t \mbox{ undefined}]\\
    &\le \sum_{t=1}^\infty \msi[a_t = \iopt, \iopt \notin F(t)\cup I(t)],
\end{align*}
which can be upper bounded by the sample complexity to determine arm $\iopt$'s feasibility, see Lemma~\ref{lemma: A_1}, thus we have
\begin{align*}
    \bbE[B_4] \le \bbE[\tau_{F,\iopt}].
\end{align*}
Putting the upper bound of $\bbE[B_3]$ and $\bbE[B_4]$ together, we can upper bound $\bbE[B_1]$.
\end{proof}

\begin{proof}[Proof of Lemma~\ref{lemma: B_2}]
    We first decompose $B_2$ into two parts $B_6, B_7$ that whether any arm in $\feasible$ has been identified as infeasible or not, we have
    \begin{align*}
        B_2 &= \sum_{s=1}^\infty \msi[t\le\tau, a_t \notin \calI, b_t \notin \calI] \\
        &\le \underbrace{\sum_{t=1}^\infty \msi[a_t \notin \calI, b_t \notin \calI, \exists i \in \feasible \text{ s.t. } i \in I(t)]}_{B_6}+\underbrace{\sum_{t=1}^\infty \msi[a_t \notin \calI, b_t \notin \calI, \forall i \in \feasible \text{ s.t. } i \notin I(t)]}_{B_7}.
    \end{align*}
    We will bound $B_6$ and $B_7$ separately by considering the sample complexity to differentiate the performance of two arms $a_t$ and $b_t$.
    
    For part $B_6$, note that the event $\{\exists i \in \feasible, s.t. i \in I(t)\}$ will only happen with probability less than $\delta$, and the event $a_t \notin \calI, b_t \notin \calI$ can be upper bound by a union bound on the sample complexity to differentiate which arm is better for all possible combinations of $a_t, b_t\notin \calI$. We have
    \begin{align*}
        B_6 &= \sum_{t=1}^\infty \msi[a_t \notin \calI, b_t \notin \calI, \exists i \in \feasible \text{ s.t. } i \in I(t)] \\
        &\le \left(\sum_{i, j \notin \calI} \sum_{t=1}^\infty \msi[a_t = i, b_t = j] \right)\cdot \msi\left[\bigcup_{t'=1}^\infty \bigcup_{k \in \feasible}\{k \in I(t')\}\right]
    \end{align*}
    We can upper bound the event $\{a_t=i, b_t = j\}$ similarly as we did to upper bound $A_5$, see Lemma~\ref{lemma: A_2}. It is easy to see that $\bbP \left(\bigcup_{t'=1}^\infty \bigcup_{k \in \feasible}\{k \in I(t')\}\right) < \delta$. Thus we can bound the expectation of $B_6$ by
    \begin{align*}
        \bbE[B_6] \le \delta \sum_{i,j \notin \calI} 2\Bigg(\frac{8}{(\Delta_{ij}/4)^2} \log\frac{8^{\frac{3}{2}}(K(N+1)/\delta)^{\frac{1}{4}}}{(\Delta_{ij}/4)^2}\log\frac{(8^{\frac{3}{2}}+1) (K(N+1)/\delta)^{\frac{1}{4}}}{(\Delta_{ij}/4)^2} + \frac{\delta}{K(N+1)}\Bigg).
    \end{align*}
    For part $B_7$, this indicates that all feasible arms are not wrongly identified as infeasible. This is the standard BAI problem. Note that even if some arms $i$ in $\{i \in \feasiblec: \theta_i > \phi_i\}$ are identified as infeasible in this scenario, they will only be eliminated faster which means $B_7$ is a bit looser for these arms. Thus we apply the result of the sample complexity of the LUCB Algorithm~\cite{kalyanakrishnan2012pac} to upper bound $B_7$. Theorem 6 from~\citet{kalyanakrishnan2012pac} indicates the expected sample complexity for the standard BAI problem using the LUCB algorithm, thus we have
    \begin{align*}
        \bbE[B_7] \le 292\sum_{i\notin \calI}\phi_i \log\frac{\sum_{i\notin \calI} \phi_i}{\delta} + 16.
    \end{align*}
    Now combining parts $B_6$ and $B_7$ together, we can upper bound the expectation of $B_2$. 
\end{proof}

\subsubsection{Proof of Lemma~\ref{lemma:tau_i} and Lemma~\ref{lemma:tau_i feasible} (\emph{Bound on Feasibility}) } 
\label{subsec:feasibility}
In this section, we will upper bound the sample complexity required to determine the feasibility of an arm. 
Let $\tau_{F,i}$ denote the number of samples on feasibility constraints to determine arm $i$'s feasibility, we have
\begin{align*}
    \tau_{F,i} = &\sum_{t=1}^\infty \msi[i \notin F(t)\cup I(t), i \mbox{ is sampled for feasibility} ] 
\end{align*}
We will first discuss the case when the arm is not in $\feasible$ in Lemma~\ref{lemma:tau_i}, 
and then discuss the case when the arm is in $\feasible$ in Lemma~\ref{lemma:tau_i feasible}. 
For this proof, we borrow some analysis techniques from~\citet{kano2019good} who study a version of the thresholding bandits problem and adapt them to our setting.


\begin{proof}[Proof of Lemma~\ref{lemma:tau_i}]
Recall that $h_i[s]$ and $H_i[s]$ denote the suggested constraint to sample and the set of undetermined constraints of arm $i$ when $\sampleforsafety$ for arm $i$ is invoked $s$ times
w.l.o.g, we assume that the feasibility constraints of arm $i$ are in decreasing order, i.e. $\mu_{i,1} > \cdots> \mu_{i,N}$. Note that it is only for the ease of notation, we will not use the ordering as part of the proof.
We can decompose $\tau_{F,i}$ as follows
\begin{align*}
    \tau_{F,i} &= \sum_{s=1}^\infty \msi \left[h_i[s] = 1, s\le \tau_{F,i} \right] + \sum_{s=1}^\infty \msi\left[h_i[s] \neq 1, s \le \tau_{F,i}\right] \\
    &\le \sum_{s=1}^\infty \msi[h_i[s] = 1] + \sum_{s=1}^\infty \msi[h_i[s] \neq 1, s \le \tau_{F,i}, \widetilde{\mu}_i^*[s] \ge \mu_{i,1} - \frac{\gamma_i(1,2)}{4}] \\
    &+ \sum_{s=1}^\infty \msi[h_i[s]\neq 1, s \le \tau_{F,i}, \widetilde{\mu}_i^*[s] < \mu_{i,1} - \frac{\gamma_i(1,2)}{4}]\\
 &\le \sum_{s=1}^\infty \msi[h_i[s] = 1] + \sum_{s=1}^{T_i} \msi[h_i[s] \neq 1, \widetilde{\mu}_i^*[s] \ge \mu_{i,1} - \frac{\gamma_i(1,2)}{4}] \\
\numberthis  &+ \sum_{s = T_i+1}^\infty \msi[s \le \tau_{F,i}] + \sum_{s=1}^\infty \msi[h_i[s] \neq 1, s\le\tau_{F,i}, \widetilde{\mu}_i^*[s] < \mu_{i,1} - \frac{\gamma_i(1,2)}{4}]. \label{eqn: summation of taui}
\end{align*}
We will bound each summation in Equation~\eqref{eqn: summation of taui} separately. Lemma~\ref{lemma: tau_i part 1 and 2} bounds the first two summations.
\begin{lemma} \label{lemma: tau_i part 1 and 2}
\begin{align}
    \nonumber &\bbE\left[\sum_{s=1}^\infty \msi[h_i[s] = 1] + \sum_{s=1}^{T_i} \msi[h_i[s] \neq 1, \widetilde{\mu}_i^*[s] \ge \mu_{i,1} -\frac{\gamma_i(1,2)}{4} ] \right] \\
    &\le n_{i,1} + \frac{8}{\gamma_i(1,1/2)^2} + \sum_{\ell =2}^N \left( \frac{8\log T_i}{\gamma_i(1,\ell)^2} + \frac{32}{\gamma_i(1,2)^2} \right)
    \label{eqn: tau_i part 1 and 2}
\end{align}
\end{lemma}
Lemma~\ref{lemma: tau_i part 3} bounds the second summation.
\begin{lemma} \label{lemma: tau_i part 3}
    \begin{align*}
        \bbE \left[ \sum_{s=T_i+1}^\infty \msi[s\le\tau_{F,i}] \right] \le  \sum_{\ell=1}^N \frac{8N(K(N+1))^{-\frac{1}{8}}}{\gamma_i(\ell,1/2)^2}.
    \end{align*}
\end{lemma}

Lemma~\ref{lemma:tau_i,part4} bounds the fourth summation.
\begin{lemma} \label{lemma:tau_i,part4}
    \begin{align*}
        &\bbE\left[ \sum_{s=1}^\infty \msi[h_i[s]\neq 1, s\le\tau_{F,i}, \widetilde{\mu}_i^*[s] < \mu_{i,1}-\frac{\gamma_i(1,2)}{4}]  \right]\\
        &\le \frac{10}{\gamma_i(1,2)^2} + \frac{16\log \frac{32}{\gamma_i(1,2)^2}}{\gamma_i(1,2)^2} +\frac{\delta}{K(N+1)} \left(\sum_{\ell \neq 1} n_{i, \ell} +\frac{8}{\gamma_i(\ell,1/2)^2}\right).
    \end{align*}
\end{lemma}

By putting together Lemma~\ref{lemma: tau_i part 1 and 2}, \ref{lemma: tau_i part 3}, and \ref{lemma:tau_i,part4}, we finish the proof to bound $\bbE[\tau_{F,i}]$ for an infeasible arm.
\end{proof}

\begin{proof}[Proof of Lemma~\ref{lemma:tau_i feasible}]
For a feasible arm, the sample complexity to determine its feasibility requires the decision maker to check each of its constraints and make sure each constraint is below the threshold $1/2$. We use Lemma~\ref{lemma: expected number about the threshold} to bound the sample complexity to determine the feasibility of each constraint.
\begin{lemma} \label{lemma: expected number about the threshold}
\begin{align*}
    &\bbE\left[ \sum_{n=1}^\infty \msi[\underline{\mu}_{i,\ell,n} \le 1/2] \right] \le n_{i, \ell} + \frac{8}{\gamma_i(\ell,1/2)^2}, \hspace{0.3in} \forall \ell: \mu_{i, \ell} > 1/2.\\
    &\bbE\left[ \sum_{n=1}^\infty \msi[\overline{\mu}_{i,\ell,n} \ge 1/2] \right] \le n_{i, \ell} + \frac{8}{\gamma_i(\ell,1/2)^2}, \hspace{0.3in} \forall \ell: \mu_{i, \ell} < 1/2.
\end{align*}
\end{lemma}
By applying Lemma~\ref{lemma: expected number about the threshold} to each constraint of arm $i$, we finish the proof to bound $\bbE[\tau_{F,i}]$ for a feasible arm. 
\end{proof}


\begin{proof}[Proof of Lemma~\ref{lemma: tau_i part 1 and 2}]
For the first term in Equation~\eqref{eqn: tau_i part 1 and 2}, we have
\begin{align}
    \sum_{s=1}^\infty \msi[h_i[s] = 1] = \sum_{s=1}^\infty \sum_{n=1}^\infty \msi[h_i[s] = 1, N_{i,1}[s] = n]  \label{eqn: tau_i part 1,1}
\end{align}
Since the event $\{h_i[s] = 1, N_{i,1}[s] = n\}$ occurs for at most one $s \in \bbN$ we have
\begin{align}
    \nonumber \sum_{s=1}^\infty \sum_{n=1}^\infty \msi[h_i[s] = 1, N_{i,1}[s] = n] &\le \sum_{n=1}^\infty \msi\left[\bigcup_{s=1}^\infty \{h_i[s] = 1, N_{i,1}[s] = n\}\right] \\
    &\le \sum_{n=1}^\infty \msi[\underline{\mu}_{i,1,n} \le 1/2]. \label{eqn: tau_i part 1,2}
\end{align}
By combining Equation~\eqref{eqn: tau_i part 1,1} and \eqref{eqn: tau_i part 1,2} with Lemma~\ref{lemma: expected number about the threshold}, we obtain
\begin{align}
    \bbE\left[ \sum_{s=1}^\infty \msi[h_i[s] = 1] \right] \le n_{i, 1} + \frac{8}{\gamma_i(1,1/2)^2}. \label{eqn: tau_i part 1, 3}
\end{align}
Next, we consider the second term in Equation~\eqref{eqn: tau_i part 1 and 2}. By using the same argument as Equation~\eqref{eqn: tau_i part 1,1} we obtain for $\ell \neq 1$ that
\begin{align*}
    &\sum_{s=1}^{T_i} \msi\left[h_i[s] = \ell, \widetilde{\mu}_i^*[s] \ge \mu_{i,1} - \frac{\gamma_i(1,2)}{4}\right] \\
    \le &\sum_{n=1}^{T_i} \msi\left[\bigcup_{s}^{T_i} \{h_i[s] = \ell, \widetilde{\mu}_i^*[s] \ge \mu_{i,1} - \frac{\gamma_i(1,2)}{4}, N_{i,\ell}[s] = n\} \right] \\
    \le &\sum_{n=1}^{T_i} \msi\left[\bigcup_{s=1}^{T_i} \{\widehat{\mu}_{i,\ell,N_{i,\ell}[s]} + \sqrt{\frac{2\log s}{N_{i,\ell}[s]}} \ge \mu_{i,1}-\frac{\gamma_i(1,2)}{4}, N_{i,\ell}[s] = n\}\right] \\
    \le &\sum_{n=1}^{T_i} \msi\left[ \widehat{\mu}_{i,\ell,n} + \sqrt{\frac{2\log T_i}{n}} \ge \mu_{i,1}-\frac{\gamma_i(1,2)}{4}\right] \\
    = &\sum_{n=1}^{T_i} \msi\left[ \widehat{\mu}_{i,\ell,n} + \sqrt{\frac{2\log T_i}{n}} \ge \mu_{i,\ell} + \gamma_i(1,\ell) -\frac{\gamma_i(1,2)}{4} \right] \\
    \le &\sum_{n=1}^{\frac{2\log T_i}{(\gamma_i(1,\ell)-\gamma_i(1,2)/2)^2}} 1 \\
    + &\sum_{n=\frac{2\log T_i}{(\gamma_i(1,\ell)-\gamma_i(1,2)/2)^2}+1}^{T_i} \msi \left[ \widehat{\mu}_{i,\ell,n} + \sqrt{\frac{2\log T_i}{\frac{2\log T_i}{(\gamma_i(1,\ell)-\gamma_i(1,2)/2)^2}}} \ge \mu_{i,\ell} + \gamma_i(1,\ell) - \frac{\gamma_i(1,2)}{4} \right] \\
    = &\frac{2\log T_i}{(\gamma_i(1,\ell)-\gamma_i(1,2)/2)^2} + \sum_{n=\frac{2\log T_i}{(\gamma_i(1.\ell)-\gamma_i(1,2)/2)^2}+1}^{T_i} \msi[\widehat{\mu}_{i,\ell,n} \ge \mu_{i,\ell} + \frac{\gamma_i(1,2)}{4}].
\end{align*}
By taking the expectation we have,
\newcommand{\hatmuiln}{\widehat{\mu}_{i,\ell,n}}
\begin{align}
    \nonumber &\bbE\left[ \sum_{s=1}^{T_i} \msi[h_i[s] = \ell, \widetilde{\mu}_i^*[s] \ge \mu_{i,1} -\frac{\gamma_i(1,2)}{4}] \right] \\
    \nonumber \le &\frac{2\log T_i}{(\gamma_i(1,\ell)/2)^2} + \sum_{n=1}^\infty \bbP\left( \hatmuiln \ge \muil + \frac{\gamma_i(1,2)}{4} \right) \\
    \nonumber \le &\frac{2\log T_i}{(\gamma_i(1,\ell)/2)^2} + \sum_{n=1}^\infty e^{-n(\frac{\gamma_i(1,2)}{4})^2/2} \hspace{1in} (\mbox{By Hoeffding's inequality}) \\
    \nonumber = &\frac{2\log T_i}{(\gamma_i(1,\ell)/2)^2} + \frac{1}{e^{(\frac{\gamma_i(1,2)}{4})^2/2}-1} \\
    \nonumber \le &\frac{2\log T_i}{(\gamma_i(1,\ell)/2)^2} + \frac{2}{(\frac{\gamma_i(1,2)}{4})^2} \\
    = &\frac{8\log T_i}{\gamma_i(1,\ell)^2} + \frac{32}{\gamma_i(1,2)^2}.   \label{eqn: tau_i part 1,4}
\end{align}
We complete the proof by combining Equation~\eqref{eqn: tau_i part 1, 3} and \eqref{eqn: tau_i part 1,4}.
\end{proof}

\begin{proof}[Proof of Lemma~\ref{lemma: tau_i part 3}]
Note that at the $s$-th time of invoking $\sampleforsafety$ of arm $i$, some constraint is pulled at least $\lceil (s-1) / N \rceil$ times. Furthermore, $N_{i, \ell}[s] \ge \lceil (s-1) / N \rceil $ implies that the constraint $\ell$ is still in $H_i[s]$ when the constraint $\ell$ is pulled $\lceil (s-1) / N \rceil-1$ times. Thus we have
\begin{align}
    \nonumber \sum_{s=T_i+1}^\infty \msi[s\le\tau_{F,i}] &\le \sum_{\ell=1}^N \sum_{s=T_i+1}^\infty \msi[N_{i, \ell}[s] \ge \lceil (s-1) / N \rceil, s\le \tau_{F,i}] \\
    \nonumber &\le \sum_{\ell:\mu_{i,\ell} > 1/2} \sum_{s=T_i+1}^\infty \msi[\underline{\mu}_{i,\ell,\lceil (s-1) / N \rceil-1} \le 1/2] \\
    &+ \sum_{\ell:\mu_{i,\ell}<1/2} \sum_{s=T_i+1}^\infty \msi[\overline{\mu}_{i,\ell,\lceil (s-1) / N \rceil-1}\ge 1/2] \label{eqn:tau_i,part3,1}
\end{align}
From the defnition of $T_i = N\max_{\ell \in [N]} \lfloor n_{i,\ell} + 2\rfloor$, we have $\lceil (s-1) / N \rceil - 1 \ge n_{i,\ell}$ for all $\ell\in[N]$. Thus the expectation of Equation~\eqref{eqn:tau_i,part3,1} is bounded by Lemma~\ref{lemma: n ge n_i} as 
\begin{align}
    \nonumber \bbE\left[ \sum_{s=T_i+1}^\infty \msi[s\le\tau_{F,i}] \right] 
    &\le \sum_{\ell=1}^N \sum_{s=T_i+1}^\infty e^{-\gamma_i(\ell,1/2)^2(\lceil (s-1) / N \rceil-1)/8} \\
     \nonumber &\le \sum_{\ell=1}^N \sum_{s=T_i+1}^\infty e^{-\gamma_i(\ell,1/2)^2( (s-1) / N -1)/8} \\
     \nonumber &= \sum_{\ell=1}^N \frac{e^{-\gamma_i(\ell,1/2)^2((T_i-1)/N -1)/8}}{e^{\gamma_i(\ell,1/2)^2/(8N)}-1} \hspace{0.3in} (\mbox{by } \sum_{i=1}^\infty e^{-ai} = \frac{1}{e^a - 1} \mbox{ for } a> 0) \\
     \nonumber &\le \sum_{\ell=1}^N \frac{e^{-\gamma_i(\ell,1/2)^2(n_{i,\ell} -2)/8}}{e^{\gamma_i(\ell,1/2)^2/(8N)}-1} \hspace{0.3in} (\mbox{by definition of }T_i) \\
     &\le \sum_{\ell=1}^N \frac{e^{-\gamma_i(\ell,1/2)^2/8 \cdot \frac{\log((K(N+1))^\frac{1}{4})}{(\gamma_i(\ell,1/2)/2)^2}}}{e^{\gamma_i(\ell,1/2)^2/(8N)}-1}  \label{eqn:tau_i,part3,2}\\ 
     \nonumber &= \sum_{\ell=1}^N \frac{8N(K(N+1))^{-\frac{1}{8}}}{\gamma_i(\ell,1/2)^2}.
\end{align}
where the Equation~\eqref{eqn:tau_i,part3,2} comes from
\begin{align*}
    n_{i,\ell} &> \frac{1}{(\gamma_i(\ell,1/2)/2)^2}\log\left(8^{\frac{3}{2}} (K(N+1))^{\frac{1}{4}}\log((8^{\frac{3}{2}}+1)(K(N+1))^{\frac{1}{4}}) \right) \\
    &\ge \frac{\log(K(N+1))^{\frac{1}{4}}}{(\gamma_i(\ell,1/2)/2)^2} + \log(8^{\frac{3}{2}}\log(8^{\frac{3}{2}}+1)) \\
    &\ge \frac{\log(K(N+1))^\frac{1}{4}}{(\gamma_i(\ell,1/2)/2)^2} + 2.
\end{align*}
\end{proof}

\begin{proof}[Proof of Lemma~\ref{lemma:tau_i,part4}]
The summation is decomposed into
\begin{align}
    \nonumber &\sum_{s=1}^\infty \msi[h_i[s]\neq 1, s\le\tau_{F,i}, \widetilde{\mu}_i^*[s] < \mu_{i,1} - \frac{\gamma_i(1,2)}{4}] \\
    &\le \sum_{s=1}^\infty \msi[\widetilde{\mu}_i^*[s]<\mu_{i,1} - \frac{\gamma_i(1,2)}{4}, 1 \in H_i[s]] + \sum_{s=1}^\infty \msi[h_i[s] \neq 1, s \le \tau_{F,i}, 1 \notin H_i[s]]. \label{eqn:tau_i,part4,1}
\end{align}
From the definition of $\widetilde{\mu}_i^*[s] = \max_{\ell \in H(s)} \widetilde{\mu}_{i,\ell}[s]$, the first term in Equation~\eqref{eqn:tau_i,part4,1} is evaluated as
\newcommand{\tildemustars}{\widetilde{\mu}_i^*[s]}
\begin{align}
    \nonumber &\sum_{s=1}^\infty \msi[\tildemustars < \mu_{i,1} - \frac{\gamma_i(1,2)}{4}, 1\in H_i[s] ] \\
    \nonumber \le &\sum_{s=1}^\infty \sum_{n=1}^\infty \msi[\widetilde{\mu}_{i,1}[s] < \mu_{i,1} - \frac{\gamma_i(1,2)}{4}, N_{i,1}[s] = n ] \\
    \nonumber = &\sum_{s=1}^\infty \sum_{n=1}^\infty \msi[\widehat{\mu}_{i, 1, n} + \sqrt{\frac{2\log s}{n}} < \mu_{i,1} - \frac{\gamma_i(1,2)}{4}, N_{i,1}[s] = n] \\
    \nonumber \le &\sum_{n=1}^\infty \sum_{s=1}^\infty \msi[s < e^{n(\widehat{\mu}_{i,1,n} - \mu_{i,1} + \gamma_i(1,2)/4)^2/2}, \; \widehat{\mu}_{i,1,n} < \mu_{i,1} - \gamma_i(1,2)/4] \\
    \le &\sum_{n=1}^\infty e^{n(\widehat{\mu}_{i,1,n} - \mu_{i,1} + \gamma_i(1,2)/4)^2/2} \msi[\widehat{\mu}_{i,1,n} < \mu_{i,1} - \gamma_i(1,2)/4]. \label{eqn:tau_i,part4,2}
\end{align}
Let $P_{i,1,n}(x) = \bbP(\widehat{\mu}_{i,1,n} < x)$. Then the expectation of the inner summation of Equation~\eqref{eqn:tau_i,part4,2} is bounded by, let $\epsilon = \gamma_i(1,2)/4$ for simple notation here,
\begin{align*}
    &\sum_{n=1}^\infty \bbE \left[e^{n(\widehat{\mu}_{i,1,n} - \mu_{i,1} + \epsilon)^2/2} \msi[\widehat{\mu}_{i,1,n} < \mu_{i,1} - \epsilon]  \right] \\
    \le &\sum_{n=1}^\infty \int_{0}^{\mu_{i,1}-\epsilon}  e^{n(\widehat{\mu}_{i,1,n} - \mu_{i,1} + \epsilon)^2/2} \rmd P_{i,1,n}(x) \\
    = &\sum_{n=1}^\infty \Bigg( \left[e^{ \frac{n(\widehat{\mu}_{i,1,n} - \mu_{i,1} + \epsilon)^2}{2}} P_{i,1,n}(x)\right]_0^{\mu_{i,1}-\epsilon} \\
    &- n\int_0^{\mu_{i,1}-\epsilon} (x-\mu_{i,1}+\epsilon)e^{ \frac{n(\widehat{\mu}_{i,1,n} - \mu_{i,1} + \epsilon)^2}{2}} P_{i,1,n}(x) \rmd x \Bigg)  \hspace{0.75in} (\mbox{by integration by parts})\\
    \le &\sum_{n=1}^\infty \left( e^{-n\epsilon^2/2} - n\int_{0}^{\mu_{i,1}-\epsilon} (x-\mu_{i,1} + \epsilon) e^{ \frac{n(\widehat{\mu}_{i,1,n} - \mu_{i,1} + \epsilon)^2}{2}} e^{-n(x-\mu_{i,1})^2/2}\rmd x \right) \hspace{0.15in}(\mbox{by Hoeffding}) \\
    = &\sum_{n=1}^\infty \left(e^{-n\epsilon^2/2} + ne^{-n\epsilon^2/2} \int_{0}^{\mu_{i,1} - \epsilon} (\mu_{i,1} - \epsilon - x) e^{n\epsilon (x-\mu_{i,1}+\epsilon)} \rmd x \right) \\
    \le &\sum_{n=1}^\infty e^{-n\epsilon^2/2} \left( 1+\frac{1}{n\epsilon^2}\right) \\
    = &\frac{1}{e^{\epsilon^2/2}-1} +  \frac{-\log(1-e^{-\epsilon^2/2})}{\epsilon^2} \hspace{1.75in}\left(\mbox{by } -\log(1-x) = \sum_{n=1}^\infty \frac{x^n}{n}\right) \\
    = &\frac{1}{e^{\epsilon^2/2}-1} + \frac{\epsilon^2/2 + \log(\frac{1}{e^{\epsilon^2/2}-1})}{\epsilon^2} \\
    \le &\frac{2}{\epsilon^2} + \frac{1}{2} + \frac{\log\frac{2}{\epsilon^2}}{\epsilon^2} \\
    \le &\frac{5}{8\epsilon^2} + \frac{\log\frac{2}{\epsilon^2}}{\epsilon^2}. \hspace{3.7in}\left(\mbox{by } \epsilon < \frac{1}{2}\right)\\
    = &\frac{10}{\gamma_i(1,2)^2} + \frac{16\log \frac{32}{\gamma_i(1,2)^2}}{\gamma_i(1,2)^2}.
\end{align*}
Combining Equation~\eqref{eqn:tau_i,part4,2} with above, we obtain
\begin{align}
    \bbE\left[ \sum_{s=1}^\infty \msi[\tildemustars < \mu_{i,1} - \gamma_i(1,2)/4, 1 \in H_i[s]] \right] \le \frac{10}{\gamma_i(1,2)^2} + \frac{16\log \frac{32}{\gamma_i(1,2)^2}}{\gamma_i(1,2)^2}. \label{eqn:tau_i,part4,3}
\end{align}
Next, we bound the second term of Equation~\eqref{eqn:tau_i,part4,1}. Note that $\{s\le\tau_{F,i}, 1 \notin H_i[s]\}$ implies that $\overline{\mu}_{i, 1} [s'] \le 1/2$ occurred for some $s' < s$. Thus we have
\begin{align}
    \nonumber &\sum_{s=1}^\infty \msi[h_i[s] \neq 1, s \le \tau_{F,i}, 1 \notin H_i[s]] \\
    \nonumber \le &\sum_{\ell \neq 1} \sum_{s=1}^\infty \msi\left[h_i[s] = \ell, \bigcup_{s'<s}\{\overline{\mu}_{i,1}[s'] \le 1/2\} \right] \\
    \nonumber \le &\sum_{\ell \neq 1}\sum_{s=1}^\infty \msi[h_i[s] = \ell] \cdot \msi \left[\bigcup_{s'=1}^\infty \{\overline{\mu}_{i,1}[s']\le 1/2\}\right] \\
    \le &\sum_{\ell \neq 1} \sum_{n=1}^\infty \msi[\overline{\mu}_{i,\ell,n} \ge 1/2] \cdot \msi \left[\bigcup_{s'=1}^\infty \{\overline{\mu}_{i,1}[s']\le 1/2\}\right] \label{eqn:tau_i,part4,4}
\end{align}
where we used the same argument as Equation~\eqref{eqn: tau_i part 1,2} in Equation~\eqref{eqn:tau_i,part4,4}. We can bound the expectation of Equation~\eqref{eqn:tau_i,part4,4} by Lemma~\ref{lemma: clean event} and Lemma~\ref{lemma: expected number about the threshold} as
\begin{align}
     \bbE\left[ \sum_{s=1}^\infty \msi[h_i[s] \neq 1, s\le\tau_{F,i}, 1\notin H_i[s]] \right] \le \frac{\delta}{K(N+1)} \left(\sum_{\ell \neq 1} n_{i, \ell} +\frac{8}{\gamma_i(\ell,1/2)^2}\right) \label{eqn:tau_i,part4,5}
\end{align}
We obtain the Lemma by putting Equation~\eqref{eqn:tau_i,part4,3} and \eqref{eqn:tau_i,part4,5} together.
\end{proof}

\begin{proof}[Proof of Lemma~\ref{lemma: expected number about the threshold}]
Note that the two cases are the same, so we will only prove for the case when $\muil>1/2$. We will split the $n \in \bbN$ to two parts $n < n_{i,\ell}$ and $n>n_{i,\ell}$. For $n>n_{i,\ell}$, Lemma~\ref{lemma: n ge n_i} upper bounds the probability that $\underline{\mu}_{i,\ell, n} \le 1/2$ and the similar case when $\muil <1/2$.
\begin{lemma} \label{lemma: n ge n_i} 
if $n\ge n_{i, \ell}$ then
\begin{align}
    &\bbP(\underline{\mu}_{i,\ell, n} \le 1/2) \le e^{-n\gamma_i(\ell,1/2)^2/8}, \hspace{0.3in} \forall \ell: \mu_{i,\ell} > 1/2, \label{eqn: ngen_i, >1/2} \\
    &\bbP(\overline{\mu}_{i,\ell, n} \ge 1/2) \le e^{-n\gamma_i(\ell,1/2)^2/8}, \hspace{0.3in} \forall \ell: \mu_{i,\ell} < 1/2. \label{eqn: ngen_i, <1/2}
\end{align}
\end{lemma}
The proof of Lemma~\ref{lemma: n ge n_i} is found in Section~\ref{subsec:feasibility}.  
Now for any constraint $\ell$ such that $\muil>1/2$, we have
    \begin{align*}
        \bbE\left[\sum_{n=1}^\infty \msi[\underline{\mu}_{i,\ell,n} \le 1/2]\right] &\le \bbE\left[\sum_{n=1}^{n_{i,\ell}}1 + \sum_{n=n_{i, \ell}+1}^\infty \msi[\underline{\mu}_{i,\ell,n} \le 1/2]\right] \\
        &\le n_{i,\ell} + \sum_{n=1}^\infty \bbP(\underline{\mu}_{i,\ell,n} \le 1/2) \\
        &\le n_{i,\ell} + \sum_{n=1}^\infty e^{-n\gamma_i(\ell,1/2)^2/8} \hspace{0.6in} (\mbox{by Lemma~\ref{lemma: n ge n_i}}) \\
        &\le n_{i,\ell} + \frac{1}{e^{\gamma_i(\ell,1/2)^2/8}-1} \\
        &\le n_{i,\ell} + \frac{8}{\gamma_i(\ell,1/2)^2}.
    \end{align*}
\end{proof}

\begin{proof}[Proof of Lemma~\ref{lemma: n ge n_i}]
We will only show Equation~\eqref{eqn: ngen_i, >1/2} for $\ell$ s.t. $\muil>1/2$ since Equation~\eqref{eqn: ngen_i, <1/2} is exactly the same. From Hoeffding's inequality, it suffices to show that for $n \ge n_{i, \ell}$,
\begin{align*}
\sqrt{\frac{2\log(4K(N+1)n^4/\delta)}{n}} \le \mu_{i,\ell} - 1/2 - \frac{\gamma_i(\ell,1/2)}{2} = \frac{\gamma_i(\ell,1/2)}{2}.
\end{align*}
Let $c = \left(\frac{\gamma_i(\ell,1/2)}{2}\right)^2$, we can express $n > n_{i, \ell}$ as 
\begin{align*}
    n = \frac{8}{c} \log \frac{8^{\frac{3}{2}}(K(N+1)/\delta)^{\frac{1}{4}}\cdot t_1}{c}
\end{align*}
for some $t_1 > \log\frac{(8^{\frac{3}{2}}+1)(K(N+1)/\delta)^{\frac{1}{4}} } {c} > 1$. Then
\begin{align*}
    &\sqrt{\frac{2\log(4K(N+1)n^4/\delta)}{n}} \le \frac{\gamma_i(\ell,1/2)}{2}.\\
    \Leftrightarrow &\sqrt{\frac{4\log(2\sqrt{K(N+1)/\delta} n^2)}{n}} \le \frac{\gamma_i(\ell,1/2)}{2}.\\
    \Leftrightarrow &\log (2\sqrt{K(N+1)/\delta}n^2) \le \frac{nc}{4}\\
    \Leftrightarrow &\log\left(2\sqrt{K(N+1)/\delta}\cdot\frac{8^2}{c^2}\left(\log\frac{8^{\frac{3}{2}}(K(N+1)/\delta)^{\frac{1}{4}}\cdot t_1}{c}\right)^2\right) \le \log \frac{8^3(K(N+1)/\delta)^{\frac{1}{2}}\cdot t_1^2}{c^2} \\
    \Leftrightarrow &\log\frac{8^{\frac{3}{2}}(K(N+1)/\delta)^{\frac{1}{4}}\cdot t_1}{c} \le 2t_1\\
    \Leftarrow &\log \frac{8^{\frac{3}{2}}(K(N+1)/\delta)^{\frac{1}{4}}}{c} + t_1 - 1 \le 2 t_1 \hspace{0.6in} (\mbox{by } \log x \le x-1)\\
    \Leftrightarrow &\log \frac{8^{\frac{3}{2}}(K(N+1)/\delta)^{\frac{1}{4}}}{ce} \le t_1.
\end{align*}
We obtain the lemma since $t_1 > \log\frac{(8^{\frac{3}{2}}+1)(K(N+1)/\delta)^{\frac{1}{4}} } {c} $.
\end{proof}





\section{Some useful results}
\label{sec:usefulresults}

\subsection{Change-of-measure Lemma}

Let $(i_r,\ell_r)$ denote tuple chose by the decision-maker at round $r$ and the reward is $X_r \in \nu_{i_r, \ell_r}$. Let $\calF_r = \sigma(i_1, \ell_1, X_1, \cdots, i_r, \ell_r, X_r)$ denote the sigma-algebra generated by $i_1, \ell_1, X_1, \cdots, i_r, \ell_r, X_r$. 

The Kullback-Leibler (KL) divergence of any two probability distributions $p$ and $q$ is defined as:
\begin{align*}
    \rmK\rmL(p,q) = 
    \begin{cases}
        \int \log [\frac{dp}{dq}(x)]dp(x) &\text{ if } q\ll p,\\
        +\infty &\text{ otherwise}.
    \end{cases}
\end{align*}
Recall that $N_{i, \ell}(r)$ denotes the number of times $(i, \ell)$ has been pulled up to round $r$. In particular, Lemma 1 in~\citet{kaufmann2016complexity} holds as:
\begin{lemma} \label{appendix: lemma: change of distribution}
\label{lem:changeofmeasure}
    Let $\nu$ and $\nu'$ be two bandit models with $K$ arms and each arm has one performance distribution and $N$ feasibility constraint distributions such that for all $i \in [K], \ell \in \{0\}\cup[N]$, the distribution $\nu_{i,\ell}$ and $\nu_{i,\ell}'$ are mutually absolutely continuous. For any almost-surely finite stopping round $\tau$ with respect to $\calF_r$,
    \begin{align*}
        \sum_{i = 1}^K \sum_{\ell=0}^N \bbE_{\nu}[N_{i,\ell}(\tau)] \rmK\rmL(\nu_{i,\ell}, \nu_{i,\ell}') \ge \sup_{\calE \in \calF_{\tau}} d(\bbP_{\nu}(\calE), \bbP_{\nu'}(\calE)),
    \end{align*}
    where $d(x,y) := x\log(x/y) + (1-x)\log((1-x)/(1-y))$ is the binary relative entropy, with the convention that $d(0,0) = d(1,1) = 0$.
\end{lemma}
We follow the remark in~\citet{kaufmann2016complexity} that
\begin{align*}
    \forall x\in[0,1], \quad d(x,1-x) \ge \log \frac{1}{2.4x}.
\end{align*}

\end{document}